\documentclass[twoside]{article}

\usepackage[accepted]{aistats2022}
\usepackage[round]{natbib}

\usepackage{nicefrac}       
\usepackage{microtype}      
\usepackage{xcolor}         
\usepackage{algorithm}
\usepackage{algorithmic}
\usepackage{graphicx}
\usepackage{subfigure}
\usepackage{amssymb}
\usepackage{amsthm}
\usepackage{amsmath}
\newtheorem{theorem}{Theorem}
\newtheorem{lemma}{Lemma}

\usepackage{mathtools}

\usepackage[citecolor=blue]{hyperref}
\newenvironment{proof*}[1][\proofname]{
  
  \begin{proof}[#1]}{\end{proof}}

\begin{document}

%

%

\twocolumn[

\aistatstitle{Regret Bounds for Expected Improvement Algorithms in Gaussian Process Bandit Optimization}

\aistatsauthor{ Hung Tran-The \And Sunil Gupta \And   Santu Rana \And Svetha Venkatesh }

\aistatsaddress{ Applied Artificial Intelligence Institute, Deakin University, Geelong, Australia } ]

\begin{abstract}
The expected improvement (EI) algorithm is one of the most popular strategies for optimization under uncertainty due to its simplicity and efficiency. Despite its popularity, the theoretical aspects of this algorithm have not been properly analyzed. In particular, whether in the noisy setting, the EI strategy with a standard incumbent converges is still an open question of the Gaussian process bandit optimization problem. We aim to answer this question by proposing a variant of EI with a standard incumbent defined via the GP predictive mean. We prove that our algorithm converges, and achieves a cumulative regret bound of $\mathcal O(\gamma_T\sqrt{T})$, where $\gamma_T$ is the maximum information gain between $T$ observations and the Gaussian process model. Based on this variant of EI, we further propose an algorithm called Improved GP-EI that converges faster than previous counterparts. In particular, our proposed variants of EI do not require the knowledge of the RKHS norm and the noise's sub-Gaussianity parameter as in previous works. Empirical validation in our paper demonstrates the effectiveness of our algorithms compared to several baselines.
\end{abstract}
\section{Introduction}
The problem of sequentially optimizing a black-box function based on bandit feedback has recently attracted a great deal of attention and finds application in robotics \citep{Lizotte07,cantin07}, environmental monitoring \citep{Marchan12}, automatic machine learning \citep{Bergstra11a,Snoek12,hoffman14} and reinforcement learning \citep{wilson14a,balakrishnan2020efficient}. Under this model, the goal is to design a sequential algorithm in a search space $\mathcal X$, i.e., a sequence $x_1, x_2 ,.., x_T$ such that at iteration $T$, the algorithm returns a state with the highest possible value. For this problem, a widely used performance measure is the \emph{cumulative regret} $R_T$, which is given by $R_T = \sum_{t= 1}^{T}\text{sup}_{x \in \mathcal X} f(x) - f(x^+_t)$, where $x^+_t$ is the point reported by the algorithm at iteration $t$.

In order to make this problem tractable, one must make smoothness assumptions on the function. A versatile means for doing this is to model the function as a Gaussian process (GP) which captures the smoothness properties through a suitably chosen kernel. For GP-based algorithms, there are two settings: \emph{Bayesian} \citep{Freitas12,tran-the21a} and \emph{non-Bayesian} \citep{scarlett17a}. In Bayesian setting, the function is assumed to be sampled from a GP while in the non-Bayesian setting, the function is treated as fixed and unknown, and assumed to lie in a reproducing kernel Hilbert space (RKHS). Under these assumptions, the optimization problem is usually called the \emph{Gaussian process bandit optimization} whereas the optimization in Bayesian setting is referred to as \emph{Bayesian optimization}. In this paper, we focus on the non-Bayesian setting, i.e. Gaussian process bandit optimization.

The Expected Improvement (EI) \citep{Mockus74} is one of the most widely used strategy to optimize black-box functions due to its simplicity and ability to handle uncertainty(e.g., works of \citet{osborne2010,wilson14a,Qin17,Malkomes2018,nguyen20d}). Unlike other popular strategies, upper confidence bound (UCB) and Thompson sampling (TS), EI is a greedy improvement-based strategy, which samples the next point offering the greater expected improvement over the current incumbent. Formally, $\alpha^{EI}(x) = \mathbb{E}[\text{max}\{0, f(x) - \xi\}]$, where $\xi$ is the incumbent to be defined. In the noiseless setting, $\xi$ is defined as the current best observation so far, $f^+_t$. Given $\mathcal D_{t}$ which is the set of sampled points up to iteration $t$, $f^+_t$ at iteration $t$ is computed as $f^+_t = \text{max}_{1 \le i \le t} f(x_i)$.
However, in the noisy setting, such a choice is not clear due to the noise.  As an alternative, $\xi$ is typically defined as either the current best value of the GP predictive mean, formally $\xi = \mu^+_t$ which is computed as $\mu^+_t = \text{max}_{1 \le i \le t} \mu_{t}(x_i)$, or the current best observation value (containing the noise), formally $\xi = y^+_t$. Although these incumbents can be easily computed, they make the theoretical analysis harder.

A key challenge of analyzing EI-based algorithms comes from its improvement function involving nonlinear, nonconvex term unlike UCB and TS. This causes the difficulty of the analysis of EI. In the noisy setting, another challenge comes from the fact that the incumbents such as $\xi = \mu^+_t$ or $\xi = y^+_t$ do not have monotonicity property like the function $f^+_t$ ($f^+_{t+1} \ge f^+_t$). This is one of crucial properties to reach the convergence in the noise-free setting \citep{Bull11}. These reasons explain why convergence properties of GP-EI are not been well studied especially in the noisy setting.

\citet{wang2014} studied GP-EI when the lower and upper bounds of hyper-parameters of GP are known. However, to guarantee the convergence, their non-peer reviewed work uses an alternative choice of the incumbent as the maximum of the GP predictive mean $\xi = \text{max}_{x \in \mathcal X} \mu_t(x)$. As a result, their GP-EI algorithm requires an additional optimization step to approximate $\mu_t(x)$ \emph{at each iteration} which is computationally expensive compared to the use of standard incumbents $\mu^+_t$ or $y^+_t$ especially when the search space is large or unbounded \citep{Tran-The_20}.

\citet{nguyen17a} proposed a ``weak” version of GP-EI which uses $y^+$ as the incumbent. This incumbent is easily computed, however their version of GP-EI needs to use an assumption that values of the variance function at all sampled points are not allowed to exceed a lower bound $\kappa$. Due to this, their GP-EI regret upper bound depends on $\kappa$ and this bound quickly explodes as $\kappa \rightarrow \infty$. As a result, their analysis does not solve the traditional EI algorithm as we consider in this paper.
Therefore, \emph{the natural question of whether GP-EI with a standard setting of incumbent (e.g., $\xi = \mu^+$ or $\xi = y^+$) converges, and if true then what convergence rate GP-EI can reach are open problems?.} In this paper, we provide an affirmative answer to these questions. Our main contributions are as follows:
\begin{itemize}
  \item We propose a variant of GP-EI  for Gaussian process bandit optimization. This algorithm uses a standard incumbent $\xi = \text{max} \{\mu_{t-1}(x_i)\}_{x_i \in \mathcal D_{t-1}}$ at iteration $t$, where $\mathcal D_{t-1}$ is the set of sampled points up to iteration $t$. Our algorithm enjoys a cumulative regret bound of $\mathcal O(\gamma_T\sqrt{T})$, where $\gamma_T$ is the maximum information gain between $T$ observations and the GP model. To our knowledge, this is the first GP-EI algorithm using a standard incumbent with theoretical guarantees.

  \item Based on our above algorithm, we propose an efficient variant of GP-EI called Improved-GP-EI of which $f$ lies in RKHS equipped by a Mat\'ern-$\nu$ kernel. Improved-GP-EI can achieve regret $\tilde{\mathcal O}(T^{\frac{d(2d+3) + 2\nu}{d(2d+4) + 4\nu}})$ for every $\nu > 1$ and $d \ge 1$. By using a searching partitioning strategy, Improved-GP-EI avoids the quick growth of the global scale of variance functions. In particular, it does not require the knowledge of the RKHS norm and the measurement noise's sub-Gaussianity parameter. These parameters are required by most previous algorithms for theoretical guarantees, but are usually unknown in real applications.

  \item We demonstrate the practical effectiveness of Improved-GP-EI against GP-EI and $\pi$-GP-UCB on various synthetic functions.
\end{itemize}

\section{Preliminaries}
We consider a global optimisation problem whose goal is to maximise $f(x)$ subject to $x \in \mathcal X \subset \mathbb{R}^d$, where $d$ is the number of dimensions and $f$ is an expensive blackbox function that can only be evaluated point-wise. The performance of a global optimisation algorithm is typically evaluated using the cumulative regret which we have defined in section Introduction.
\subsection{Regularity Assumptions}
We assume that $f$ lives in a RKHS of functions $\mathcal X \rightarrow \mathbb{R}$ with positive semi-definite kernel function $k$. This $H_k$ space is defined as the Hilbert space of functions on $\mathcal X$ equipped with an inner product $\langle .\rangle_k$  obeying the reproducing property: $f(x) = \langle f, k(x,.) \rangle_k$ for all $f \in H_k(\mathcal X)$. The RKHS norm $||f||_k = \sqrt{\langle f, f \rangle_k}$ is a measure of smoothness of $f$, with respect to the kernel function $k$, and satisfies: $f \in H_k(\mathcal X)$ if and only if $||f||_k < \infty$. We assume that the RKHS norm of the unknown target function is bounded by $||f||_k \le B$. Two common kernels that satisfy
bounded variance property are Squared Exponential (SE) and Mat\'ern, defined as
\begin{eqnarray}
k_{\text{SE}}(x, x') = \text{exp}(\frac{-||x-x'||^2}{2l^2}),
\label{eq_kernel}
\end{eqnarray}
\begin{eqnarray}
k_{\text{Mat}}(x, x') = \frac{2^{1-\nu}}{\Gamma(\nu)}(\frac{||x -x'||_2}{l})^{\nu}\mathcal B_{\nu}(\frac{||x- x'||_2}{l}),
\label{eq_kernel}
\end{eqnarray}
where $\Gamma$ denotes the Gamma function, $\mathcal{B}_{\nu}$ denotes the modified Bessel function of the second kind, $\nu$ is a parameter controlling the smoothness of the function and $l$ is the lengthscale of the kernel. Important special cases of $\nu$ include $\nu = \frac{1}{2}$ that corresponds to the exponential kernel and $\nu \rightarrow \infty$ that corresponds to the square exponential (SE) kernel. The Mat\'ern kernel is of particular practical significance, since it offers a more suitable set of assumptions for the modeling and optimisation of physical quantities \citep{stein1999}.
\subsection{Gaussian process bandit optimization}
Gaussian process bandit optimization proceeds sequentially in an iterative fashion. At each iteration, a surrogate model is used to probabilistically model $f(x)$. Gaussian process (GP) \citep{Rasmussen05} is a popular choice for the surrogate model as it offers a prior over a large class of functions and its posterior and predictive distributions are tractable. Formally, we assume $f(x) \sim \mathcal {GP}(m(x), \omega^2 k(x, x'))$ a prior distribution where $m(x)$ is the mean function and $\omega^2 k(x, x')$ is the covariance function in which $k(x, x')$ is a kernel function associated with the RKHS $H_k$ in which $f$ is assumed to have norm at
most $B$, and $\omega > 0$ is a parameter to capture the global scale of variation of function $f$. Without loss of generality, we assume that $m(x) = 0$. Given a set of observations $\mathcal D_{1:t} = \{x_i, y_i\}_{i=1}^t$, the predictive distribution can be derived as $P(f_{t+1}| \mathcal D_{1:t}, x) = \mathcal N(\mu_{t+1}(x), \omega^2\sigma_{t+1}^2(x))$, where $$\mu_{t+1}(x) = k_t(x)^T[K_t + \lambda I]^{-1}y_{1:t},$$
$$\sigma_{t+1}^2(x) = k(x,x) - k_t(x)^{T}(K_t + \lambda I)^{-1}k_t(x),$$
where we define $k_t(x)= [k(x, x_1), ..., k(x, x_t)]^T$, $K_t= [k(x_i, x_j)]_{1 \le i,j \le t}$,  $y_{1:t}=[y_1,\ldots,y_t]$ and $\lambda$ as variance of the measurement noise.

We assume that kernel function $k$ is fixed and known and, without loss of generation, the variance of $k$ is bounded as $k(x,x) \le 1$. These assumptions are similar as in \citep{Srinivas12,Chowdhury17,janz20a}. However, unlike these works, we do not require the knowledge of the sub-Gaussianity parameter $R$ and upper bound $B$ on the RKHS norm of $f$.

In our setting, we note that the parameters $\omega$ and $\lambda$ are possibly time-dependent. They can be set specific to an algorithm as in many previous works (e.g., \citep{Bull11}, \citep{agrawal13}, \citep{wang2014}, \citep{Chowdhury17}).
\subsection{Expected Improvement}
An acquisition function is used to suggest the point $x_{t+1}$ where the function should be next evaluated. The acquisition step uses the predictive mean and the predictive variance from the surrogate model to balance the exploration of the search space and exploitation of current promising region. Some examples of acquisition functions include GP-EI \citep{Bull11}, GP-UCB \citep{Srinivas12}, GP-TS \citep{Chowdhury17}, and entropy based methods e.g., PES \citep{Lobato14}.

In the noisy case, the function is evaluated as $y_t = f(x_t) + \epsilon_t$, which is a noisy version of the function value at $x_t$. We assume that the noise sequence $\{\epsilon_t\}_{t=1}^{\infty}$ is conditionally $R$-sub-Gaussian for a fixed constant $R \ge 0$, i.e.,
$\forall t \ge 0, \forall \lambda \in \mathbb{R}, \mathbb{E}[e^{\lambda\epsilon_t}| \mathcal F_{t-1}] \le \text{exp}(\frac{\lambda^2R^2}{2})$, where $\mathcal F_{t-1}$ is the $\sigma$-algebra generated by the random variables. This is a mild assumption on the noise and is standard in the BO literature \citep{Chowdhury17} and also in bandit literature \citep{Abbasi11}.

We let $\mathcal D_{t} = \{x_1 ,..., x_t\}$ denote the set of chosen points to be evaluated up to iteration $t$. The noise in the evaluation of the incumbent causes it to be brittle. A standard choice of the incumbent in this setting is the best value of the GP mean function so far $\mu^+_t = \text{max} \{\mu_{t-1}(x_i)\}_{x_i \in \mathcal D_t}$. We note that the maximization is only over the observed points, not the complete input space $\mathcal X$. For this choice, EI is written in closed form as:
\begin{eqnarray}
\alpha_t^{EI}(x) & = & \mathbb{E}[\text{max}\{0, f(x) - \mu^+_{t}\}|\mathcal D_t] \\
& = & \rho(\mu_{t-1}(x) - \mu^+_{t}, \omega \sigma_t(x)),
\end{eqnarray}
where $\rho(u,v)$ with two arguments $u$ and $v$ is defined as
\begin{equation}
  \rho(u,v)=\begin{cases}
    u\Phi(\frac{u}{v}) + v\phi(\frac{u}{v}), & \text{if $ v > 0$},\\
    \text{max} \{0, u\}, & \text{if $ v= 0$},
  \end{cases}
\label{eq1}
\end{equation}
and $\Phi$ and $\phi$ are the standard normal distribution and density functions respectively.
\section{Gaussian Process Expected Improvement (GP-EI) Algorithm}
\begin{algorithm}[tb]
\caption{GP-EI Algorithm}
\label{algo1}
\textbf{Input}: Prior $\text{GP}(0, k)$
\begin{algorithmic}[1]
\FOR{$t =1$ to $T$}
    \STATE Set $\omega_t = \sqrt{\gamma_{t-1} + 1 + \text{ln}(\frac{1}{\delta})}$
    \STATE Choose the next sampling point: $x_t = \text{argmax}_{x \in \mathcal X} \rho(\mu_{t-1}(x) - \mu_{t}^+, \omega_t \sigma_{t-1}(x))$
    \STATE Observe $y_t = f(x_t) + \epsilon_t$
    \STATE Update Gaussian process
\ENDFOR
\end{algorithmic}
\textbf{Output}: Report points $x^+_t = \text{argmax}_{1\le i \le t} \mu_{t-1}(x_i)$ for all $1 \le t \le T$
\end{algorithm}
Our GP-EI algorithm is represented in Algorithm \ref{algo1}. The time-varying scale parameter $\omega_t = \sqrt{\gamma_{t-1} + 1 + \text{ln}(1/\delta))}$ is used to control the exploration of the algorithm for guaranteeing the convergence. Here $\delta$ is a free parameter in $(0,1)$ and $\gamma_{t}$ is the maximum information gain at time $t$ which is defined as $\gamma_t = \text{max}_{A \in \mathcal X: |A| = t} I(y_A;f_A)$, where $I(y_A;f_A)$ denotes the mutual information between $f_A = [f(x)]_{x \in A}$ and $y_A = f_A + \epsilon_A$ and $\epsilon_A \in \mathcal N(0, \lambda\omega^2_tI)$. The algorithm choose a point $x_t = \text{argmax}_{x \in \mathcal X} \alpha_t^{EI}(x)$ which is computed by Eq (\ref{eq1}) to sample. After $T$ iterations, the algorithm returns points $x^+_t = \text{argmax}_{1\le i \le t} \mu_{t-1}(x_i)$ for every $1 \le t \le T$.

We note that the time-varying scale of variation of function $f$ has been usually utilized by many previous works e.g. \citep{Chowdhury17} which analyzes the GP-TS algorithm, \citep{agrawal13} which analyzes a Thompson Sampling algorithm but for the contextual bandit problem, and \citep{wang2014} which analyzes the EI algorithm. For example, in the setting of \citep{Chowdhury17}, they used a Gaussian process with mean 0 and variance $v^2_tk(.,.)$, where $v_t$ is a global scale parameter of the variance which is allowed to vary with time. For their GP-Thompson sampling, the time-varying scale parameter is set as $v_t = B + R \sqrt{2(\gamma_{t-1} + 1 + \text{ln}(2/\delta))}$ (See section 3.2 therein).
\paragraph{Comparison with related algorithms.} Most of previous algorithms for Gaussian process bandit optimization such as GP-UCB \cite{Srinivas12}, Improved-GP-UCB \citep{Chowdhury17}, $\pi$-GP-UCB \citep{janz20a}, and GP-TS \citep{Chowdhury17} require to know exactly the sub-Gaussianity parameter $R$ and upper bound $B$ on the RKHS norm of $f$ so that theoretical convergence guarantees hold. However, these parameters are often unknown in real applications. To overcome this issue, as an example, \citet{Berkenkamp19} proposed to learn unknown $B$ by starting from an initial guess $B_0$  and then scale up the norm bound over time. As a result, $B$ is replaced by a time-varying function
$B_0b(t)g(t)^d$ in their Theorem 1, where functions $b(t)$ and $g(t)$ are designed heuristically. As a result, this causes an additional $\mathcal O(b(t)g(t)^{3d/2})$ factor in the regret. Unlike these algorithms, our GP-EI algorithm can avoid the need to specify or learn such parameters. This is because our algorithm uses $\omega_t = \sqrt{\gamma_t + 1 + ln(\frac{1}{\delta})}$ which is independent of $B$ and $R$.
\subsection{Theoretical Result}
Importantly, we achieve a cumulative regret bound for the proposed GP-EI algorithm, denoted by $R_T = \sum_{t=1}^{T} (f(x^*) - f(x^+_t))$ as follows:
\begin{theorem}
Pick $\delta \in (0,1)$. Then with probability at least $1 - \delta$, the cumulative regret of Algorithm \ref{algo1} is bounded as:
$$R_T = \mathcal O(\gamma_T\sqrt{T}).$$
\label{theorem2}
\end{theorem}
The complete form of $R_T = \mathcal O(\beta_T\sqrt{T\gamma_T})$, where $\beta_T = B + R \sqrt{2(\gamma_{t-1} + 1 + \text{ln}(1/\delta))}$ is provided in Supplementary Material. Here we remove the influence of constants $B, R$ for simplicity. The regret bound of our GP-EI algorithm is same as that of Improved-GP-UCB \citep{Chowdhury17} but improve GP-TS \citep{Chowdhury17} by a factor $\sqrt{d\text{ln}(dT)}$. For SE kernels, $R_T$ is sublinear on $T$.  However, for Mat\'ern kernels, this proposed algorithm still has some limitations. First, its regret bound  is not always sublinear in $T$. \citet{vakili21a} currently provides a new bound for $\gamma_T$ as $\gamma_T = \tilde{\mathcal O}(T^{\frac{d}{2\nu + d}})$. By this, $2\nu > d$ is required so that our proposed GP-EI algorithm obtains a sublinear regret. Second,$\gamma_t$ of the proposed GP-EI algorithm grows quickly with $t$. In practice, it can cause unnecessary explorations. These motivate us to propose a new variant of the GP-EI algorithm in the next section.
\subsection{The Improved-GP-EI Algorithm}
In this section, we propose a new variant of GP-EI, called Improved-GP-EI, inspired from the $\pi$-GP-UCB algorithm \citep{janz20a}. Improved-GP-EI uses a global scale $\omega_T = \sqrt{\text{ln}(T)\text{ln} \text{ln}(T)}$ growing only poly-logarithmically with $T$. Here we assume that $T$ is known as in \citep{janz20a} and we use the same $\omega_T$ at all iterations from 1 to $T$. Improved-GP-EI is an adaptation of $\pi$-GP-UCB algorithm \citep{janz20a} to GP-EI. The key difference is that (1) our Improved GP-EI uses EI acquisition function instead of UCB, and (2) Improved GP-EI use a new time-varying scale (in $T$) parameter $\omega_T$ instead of a constant like \citep{janz20a}. Now we start to describe the underlying idea of Improved GP-EI.
\paragraph{Improved-GP-EI algorithm}
At each iteration $t$, the algorithm constructs a cover (a set of hypercubes) of domain $\mathcal X$, $\mathcal A_t$, and selects a point $x_t$ to be evaluated in the next iteration by taking a maximizer of the GP-EI constructed independently on each cover element. The cover $\mathcal A_t$ is constructed by induction starting from the initial cover $\mathcal A_1$. Similar to $\pi$-GP-UCB,  we set $b = \frac{d+1}{d + 2\nu}$ and $q = \frac{d(d+1)}{d(d+2) + 2\nu}$. For a hypercube $A \in \mathcal X$, we will use $\rho_A$ to denote its diameter.

Let $\mathcal A_1$ be any set of closed hypercubes of cardinality at most $\mathcal O(T^q)$ overlapping at edges only and covering the domain $\mathcal X$. At each iteration $t$, we build a GP, select the next point $x_t$ and then construct a new cover $\mathcal A_{t+1}$ as below:
\begin{itemize}
  \item \textbf{Update GP:} Fit an independent GP on each cover element $A \in \mathcal A_t$, \emph{using only the data within $A$}. We define the subset of $\mathcal D_t$ in $A$ as $\mathcal D^A_t = \{x^A_i \in \mathcal D_t| x_i^A \in A\}$. For $\mathcal D^A_t$ of cardinality $N$, we define the kernel $k$ as $k^A_t = [k(x^A_1, x), ..., k(x^A_N, x)]^T$ and the kernel matrix as $K^A_t = [k(x, x')]_{x, x' \in \mathcal D^A_t}$. We use $y^A_{1:t}$ to denote the observations corresponding to points in $\mathcal D^A_t$. For a regularisation parameter $\lambda > 0$, we define the Gaussian process on $A \in \mathcal X$ by mean, $$\mu_t^A = k_t^A(x)(K^A_t + \lambda I)^{-1} y^A_{1:t},$$
         and predictive standard deviation,
        $$\sigma^A_t(x) = \sqrt{k(x,x) - k_t^A(x^T(K_t^A + \lambda I)^{-1}) k_t^A(x)}.$$
   \item \textbf{Next point selection:}  We select the next point to evaluate as  $x_t = \text{argmax}_{A \in \mathcal A_t: x \in A}\alpha_t^{EI} =\text{argmax}_{A \in \mathcal A_t: x \in A} \rho(\mu_{t-1}(x) - \mu^{A+}_{t-1}, \omega_T \sigma^A_{t-1}(x)),$
      where $\mu^{A+}_{t-1}$ is defined as $\mu^{A+}_{t-1} = \text{max}_{x_i \in A \cap \mathcal D_{t-1}} \mu^A_{t-1} (x_i)$, and function $\rho$ is defined by Eq(\ref{eq1}).
  \item \textbf{Build $\mathcal A_{t+1}$:} Split any element $A \in \mathcal A_{t}$ for which $\rho_A^{-1/b} < |\mathcal D^A_{t+1}| +1$ along the middle of each side, resulting in $2^d$ new hypercubes. Let $\mathcal A_{t+1}$ be the set of the newly created hypercubes and the elements of $\mathcal A_{t}$ that were not split. (See \citep{janz20a} for details.)
\end{itemize}
We now present the regret bound for our Improved GP-EI algorithm.
\begin{theorem}
Pick $\delta \in (0,1)$.  Let $H_k(\mathcal X)$ be the RKHS
of a Mat\'ern kernel $k$ with parameter $\nu > 1$. Then with probability at least $1 - \delta$, the cumulative regret of Improved GP-EI has the following rate:
$$R_T = \tilde{\mathcal O}(T^{\frac{d(2d+3) + 2\nu}{d(2d+4) + 4\nu}}).$$
\label{theorem3}
\end{theorem}
The regret bound of Improved GP-EI is sublinear in $T$ for every $\nu > 1$, thus improves over that of our above proposed GP-EI algorithm. Furthermore, we use $\omega_T = \sqrt{\text{ln}(T)\text{ln} \text{ln}(T)}$ which grows only poly-logarithmically with $T$. This property is useful in practice as it avoids unnecessary explorations. Finally, we note that our Improved GP-EI achieves the same regret rate as $\pi$-GP-UCB \citep{janz20a} on the regret, however it does not require to know parameters $B, R$ like the work of \citet{janz20a}.
\subsection{Proof Sketch for Theorem \ref{theorem2}}
In this section, we provide the proof sketch for Theorem \ref{theorem2}. A complete proof of Theorem \ref{theorem2} is provided in Appendix B. To bound the cumulative regret $R_T$, we proceed to bound instantaneous regrets $r_t = f(x^*) - f(x_t^+)$. The proof involves two steps as follows.
\paragraph{Upper bounding the instantaneous regret $r_t = f(x^*) - f(x_t^+)$:} We break down $r_t$ into two terms as follows:
\begin{eqnarray*}
r_t  =  f(x^*) - f(x_t^+)  =   \underbrace{f(x^*) - \mu^+_{t}}_{\text{Term 1}} + \underbrace{\mu_{t}^+ - f(x_t^+)}_{\text{Term 2}}
\end{eqnarray*}
Set $I_t = \text{max}\{0, f(x_{t+1}) - \mu_t^+\}$. We upper bound Term 1 through the following lemma:
\begin{lemma}
Pick $\delta \in (0,1)$. Then with probability at least $1 -\delta$ we have
$$ f(x^*) - \mu_{t}^+ \le  \frac{\tau(\frac{\beta_t}{\omega_t})}{ \tau(-\frac{\beta_t}{\omega_t})}( I_t + (\beta_t + \omega_t)\sigma_{t}(x_{t+1})),$$
where given any $z \in \mathbb{R}$, the function $\tau(z)$ is defined as $\tau(z) = z\Phi(z) + \phi(z)$, where $\Phi$ and $\phi$ are the standard normal distribution and density functions respectively.
\end{lemma}
We upper bound Term 2 through the following lemma:
\begin{lemma}
Pick a $\delta \in (0,1)$. Then with probability $1 - \delta$ we have $$\mu_t^+ - f(x^+_t) \le \frac{\beta_t}{\omega_t} (\sqrt{2\pi}(\beta_t + \omega_t)\sigma_{t}(x_{t+1}) + \sqrt{2\pi} I_t).$$
\end{lemma}

By using $\omega_t = \sqrt{\gamma_{t-1} + 1 + ln(\frac{1}{\delta})}$, we obtain $\frac{\beta_t}{\omega_t} \le B + \sqrt{2}$ and there exists a constant $C > 0$ such that $\frac{\tau(\frac{\beta_t}{\omega_t})}{ \tau(-\frac{\beta_t}{\omega_t})} \le C$  for every $t$. Combining these results,  we obtain an upper bound for the regret $r_t$ as $r_t \le (B + C + 2)( I_t + 2\beta_t \sigma_{t}(x_{t+1}))$, where $C > 0$ is constant.

To achieve Lemma 1, we adapt several results of \citep{Bull11} in noise-free setting to our the noisy setting, and to achieve Lemma 2, we exploit additionally properties of the function $\tau(z)$ and the points $x^+_t = \text{argmax}_{1 \le i \le t}\mu_{t-1}(x_i)$. We note that in the noise-free setting, we can use $\xi = f^+$ as the incumbent in the form of the expected improvement $\alpha^{EI}(x) = \mathbb{E}[\text{max}\{0, f(x) - \xi\}]$, where $f^+_t$ is the current best observed function value so far. In the noisy setting, the function values can not observed due to noises. Using $\xi = \mu^+_t$ as a replacement allows to compute easily the incumbent but also causes the difficulty in the theoretical analysis. While \citet{Bull11} leverages the \emph{monotonicity} of $f^+_t = \text{max}_{1 \le i \le t} f(x_i)$ to derive directly an upper bound for the regret $r_t$, this is very challenging in our setting because the values $\mu^+_i$ with $1 \le i \le t$ have no monotonicity property. To overcome this, we seek to upper bound the sum of $\sum_{t=1}^{T} r_t$. We obtain $\sum_{t=1}^{T} r_t \le (C + B + 2)(\sum_{t=1}^{T} I_t + \sum_{t=0}^{T-1} \beta_t \sigma_{t}(x_{t+1}))$.

While upper bounding $\sum_{t=0}^{T-1}\sigma_{t}(x_{t+1})$ can be achieved via the maximum information gain like previous works in the noisy setting \citep{Srinivas12,Chowdhury17}, upper bounding the sum $\sum_{t=0}^{T-1} \text{max}\{0, f(x_{t+1}) - \mu_t^+\}$ is the \emph{key challenge} in our regret analysis.
We overcome this difficulty by exploiting the \emph{monotonicity of variance functions} \citep{vivarelli} which shows that $\sigma_{t} (x) \ge \sigma_{t'} (x)$ if $t \le t'$. To our knowledge, we exploit for the first time this property for Gaussian process bandit optimization problem.

We achieve an upper bound for $\sum_{t=0}^{T-1} \text{max}\{0, f(x_{t+1}) - \mu_t^+\}$ as in the following important lemma.
\begin{lemma}
Pick a $\delta \in (0,1)$. Then with probability at least $1 - \delta$ we have that
$$\sum_{t=0}^{T-1} \text{max}\{0, f(x_{t+1}) - \mu_t^+\} \le \mathcal O(\beta_T \sqrt{T\gamma_T}),$$
where $\beta_T = B + R\sqrt{2(\gamma_{T-1} + 1 + ln(1/\delta))}$.
\end{lemma}
To prove this lemma, we need two auxiliary lemmas from the literature.

\begin{lemma}[Theorem 2 of \citet{Chowdhury17}]
Pick $\delta \in (0,1)$. We define $\beta_t = B + R\sqrt{2(\gamma_{t-1} + 1 + ln(1/\delta))}$ for every $1 \le t \le T$. Then $\mathbb{P}(\forall 1 \le t \le T, \forall x \in \mathcal X, |f(x) - \mu_{t-1}(x)| \le \beta_t \sigma_{t-1}(x) ) \ge 1 - \delta$.
\label{le_a}
\end{lemma}

\begin{lemma}[Lemma 5 of \citet{Freitas12}]
When $f \in H_k(\mathcal X)$, then for every $x, y \in \mathcal X$, we have $|f(x) - f(y)| \le BL||x -y||_1$,
where $L$ is the Lipschitz constant in $H_k(\mathcal X)$.
\label{le_b}
\end{lemma}

\paragraph{Proof of Lemma 3} Set $S_T = \sum_{t=0}^{T-1} I_t$. There are three cases to be considered:
\paragraph{Case 1} $S_T = 0$. This happens when for every $t$: $f(x_{t+1}) - \mu_t^+ \le 0$.
\paragraph{Case 2} There exists an unique index $1 \le t'\le T$ such that $f(x_{t'+1}) - \mu_t'^+ > 0$. It follows that $S_T = f(x_{t'+1}) - \mu_t'^+$. In this case, we have that
\begin{eqnarray*}
S_T  & = & f(x_{t'+1}) - \mu_t'^+ \\
& \le & f(x_{t'+1}) - (f(x') - \beta_{t'+1}\sigma_{t'}(x')) \\
& \le & f(x_{t'+1}) - f(x') + \beta_{t'+1}\sigma_{t'}(x') \\
& \le &  BL ||x_{t' + 1} - x'||_1 + \beta_{t'+1}\\
& = & \mathcal O(\beta_T),
\end{eqnarray*}
where in the last inequality, we use Lemma \ref{le_b}, the inequality $\beta_{t'+1} \le \beta_T$, and the fact that $\sigma_t'(x) \le 1$. Finally, because the domain $\mathcal X$ is bounded, $||x_{t' + 1} - x'||_1$ is bounded.

\paragraph{Case 3} There are $0 \le t_1 < t_2, ..., < t_l \le T-1$ where $l \ge 2$ such that $f(x_{t_i +1}) \ge \mu_{t_i}^+$. Thus, we have
\begin{eqnarray*}
\sum_{t=0}^{T-1} I_t & = & \sum_{t=1}^{T} \text{max}\{0, f(x_{t+1}) - \mu_t^+\} \\
& = & \sum_{i = 1}^{l} (f(x_{t_i +1}) - \mu_{t_i}^+) \\
& \le & \sum_{i = 1}^{l} (\beta_{t_i+1} \sigma_{t_i}(x_{t_i +1}) + \mu_{t_i}(x_{t_i +1}) - \mu_{t_i}^+) \\
& \le & \underbrace{\sum_{i = 1}^{l} \beta_{t_i+1} \sigma_{t_i}(x_{t_i +1})}_{\text{Term 5}} +  \underbrace{\sum_{i = 1}^{l}(\mu_{t_i}(x_{t_i +1}) - \mu_{t_i}^+)}_{\text{Term 6}}
\end{eqnarray*}
\paragraph{Bound Term 5}
\begin{eqnarray*}
\sum_{i = 1}^{l} \beta_{t_i+1} \sigma_{t_i}(x_{t_i +1}) & \le & \sum_{t=0}^{T-1} \beta_{t+1} \sigma_{t}(x_{t+1}) \\
& \le &  \beta_T \sum_{t=0}^{T-1} \sigma_{t}(x_{t+1})
\end{eqnarray*}
\paragraph{Bound Term 6} Set $M_1 = \sum_{i = 1}^{l}(\mu_{t_i}(x_{t_i +1}) - \mu_{t_i}^+)$.
\begin{eqnarray*}
& M_1 & =\mu_{t_l}(x_{t_l +1}) - \mu_{t_1}^+  + \sum_{i=1}^{l-1} (\mu_{t_{i-1}}(x_{t_{i-1} +1}) - \mu_{t_i}^+))\\
&\le & \underbrace{\mu_{t_l}(x_{t_l}) - \mu_{t_1}^+}_{\text{Term 7}}  + \underbrace{\sum_{i=1}^{l-1} (\mu_{t_{i-1}}(x_{t_{i-1} +1}) - \mu_{t_i}(x_{t_{i-1} +1})))}_{\text{Term 8}}\\
\end{eqnarray*}
\paragraph{Bound Term 7} Set $M_2 = \mu_{t_l}(x_{t_l+1}) - \mu_{t_1}^+ $. We have
\begin{eqnarray*}
M_2 & \le & f(x_{t_l + 1}) + \beta_{t_l+1}\sigma_{t_l}(x_{t_l+1}) - (f(x_{t_1} - \beta_{t_1}\sigma_{t_1}(x_{t_1}))) \\
& \le &  f(x_{t_l + 1}) - f(x_{t_1}) + \beta_{t_l+1}\sigma_{t_l}(x_{t_l+1})+ \beta_{t_1 +1}\sigma_{t_1}(x_{t_1})) \\
& \le & f(x_{t_l + 1}) - f(x_{t_1}) + \beta_{t_l +1}  + \beta_{t_1}  \\
& \le & BL||x_{t_l + 1} - x_{t_1}||_1 + 2\beta_T  \\
& \le & \mathcal O(\beta_T)
\end{eqnarray*}
The argument to achieve the bound for Term 7 is similar to Case 2.
\paragraph{Bound Term 8}
Set $M_3 = \sum_{i=1}^{l-1} (\mu_{t_{i-1}}(x_{t_{i-1} +1}) - \mu_{t_i}(x_{t_{i-1} +1})))$ for simplicity. We go to bound $M$.
\begin{eqnarray*}
& M_3 &  \le   \sum_{i=1}^{l-1} (f(x_{t_{i-1} +1}) + \beta_{t_{i-1}+1}\sigma_{t_{i-1}}(x_{t_{i-1} +1}))  \\
&  & - (f(x_{t_{i-1} +1})- \beta_{t_{i}+1}\sigma_{t_{i}}(x_{t_{i-1} +1})) \\
& = & \sum_{i=1}^{l-1} \beta_{t_{i-1}+1}\sigma_{t_{i-1}}(x_{t_{i-1} +1}) + \beta_{t_{i}+1}\sigma_{t_{i}}(x_{t_{i-1} +1}) \\
& \le & \sum_{i=1}^{l-1}(\beta_{t_{i-1}+1} + \beta_{t_{i}+1})\sigma_{t_{i-1}}(x_{t_{i-1} +1})\\
& \le & 2\beta_T \sum_{i=1}^{l-1}\sigma_{t_{i-1}}(x_{t_{i-1} +1})\\
& \le & 2\beta_T \sum_{i=0}^{T-1} \sigma_{i}(x_{i+1}),
\end{eqnarray*}
where in the first inequality, we use Lemma \ref{le_a}: $\mu_{t_{i-1}}(x_{t_{i-1} +1}) \le f(x_{t_{i-1} +1}) + \beta_{t_{i-1}}\sigma_{t_{i-1}}(x_{t_{i-1} +1})$; $\mu_{t_i}(x_{t_{i-1} +1}) \ge f(x_{t_{i} +1}) - \beta_{t_{i}}\sigma_{t_{i}}(x_{t_{i-1} +1})$. In the second inequality, we use the fact that $f(x_{t_{i-1} +1}) \le f(x_{t_{i} +1})$. In the third inequality, we use the \emph{decreasing monotonicity of variance functions} (\citep{vivarelli} and \citep{Chowdhury17}, see Section F). Here, we use $$\sigma_{t_{i}}(x_{t_{i-1} +1}) \le \sigma_{t_{i-1}}(x_{t_{i-1} +1}),$$ because $T-1 \ge t_i > t_{i-1} \ge 0$ due to the definition of $t_i$ and $t_{i-1}$.
This step is crucial to bound $M_3$. Without this step, $M_3$ may be bounded by two sums:  $\sum^{l-1}_{i=1}  \sigma_{t_{i-1}}(x_{t_{i-1} +1}))$ and $ \sum^{l-1}_{i=1} \sigma_{t_{i}}(x_{t_{i-1} +1}))$. While the first term can be bounded in terms of the information gain, bounding the second is challenging, and was what led to an error in Lemma 7 of \cite{nguyen17a}.

For every $x_i$, where $ 1 \le i \le T-1$, Lemma \ref{le_a} holds with probability $1 -\delta$. Therefore, Lemma \ref{le_a} holds with probability at least $1 -\delta$ for all $x_i$, where $ 1 \le i \le T-1$.
Combining Term 5, Term 7, Term 8, with probability $1 -T\delta$ we have
$$\sum_{t=0}^{T-1} \text{max}\{0, f(x_{t+1}) - \mu_t^+\} \le  \mathcal O(\beta_T \sum_{i=1}^{T-1} \sigma_{i}(x_{i+1})).$$

On the other hand, following Lemma 4 of \cite{Chowdhury17}, we have $\sum_{i=1}^{T-1} \sigma_{i}(x_{i+1}) \le  \sqrt{4(T+2)\gamma_T}$.  Thus,
$$\sum_{t=0}^{T-1} \text{max}\{0, f(x_{t+1}) - \mu_t^+\} = \mathcal O(\beta_T \sqrt{T\gamma_T}).$$

Thus, for all cases, Lemma 3 holds.
\subsection{Proof Sketch for Theorem 2}
Theorem \ref{theorem3} holds using a non-trivial combination of the proof techniques as above and the technical results for $\pi$-GP-UCB \cite{janz20a}. A complete proof of Theorem \ref{theorem3} is provided in Appendix C of the Supplementary Material.

\section{Related works}
\begin{figure*}[h]
\centering
\subfigure{\includegraphics[scale=1.0,width=.225\textwidth, height=.11\textheight]{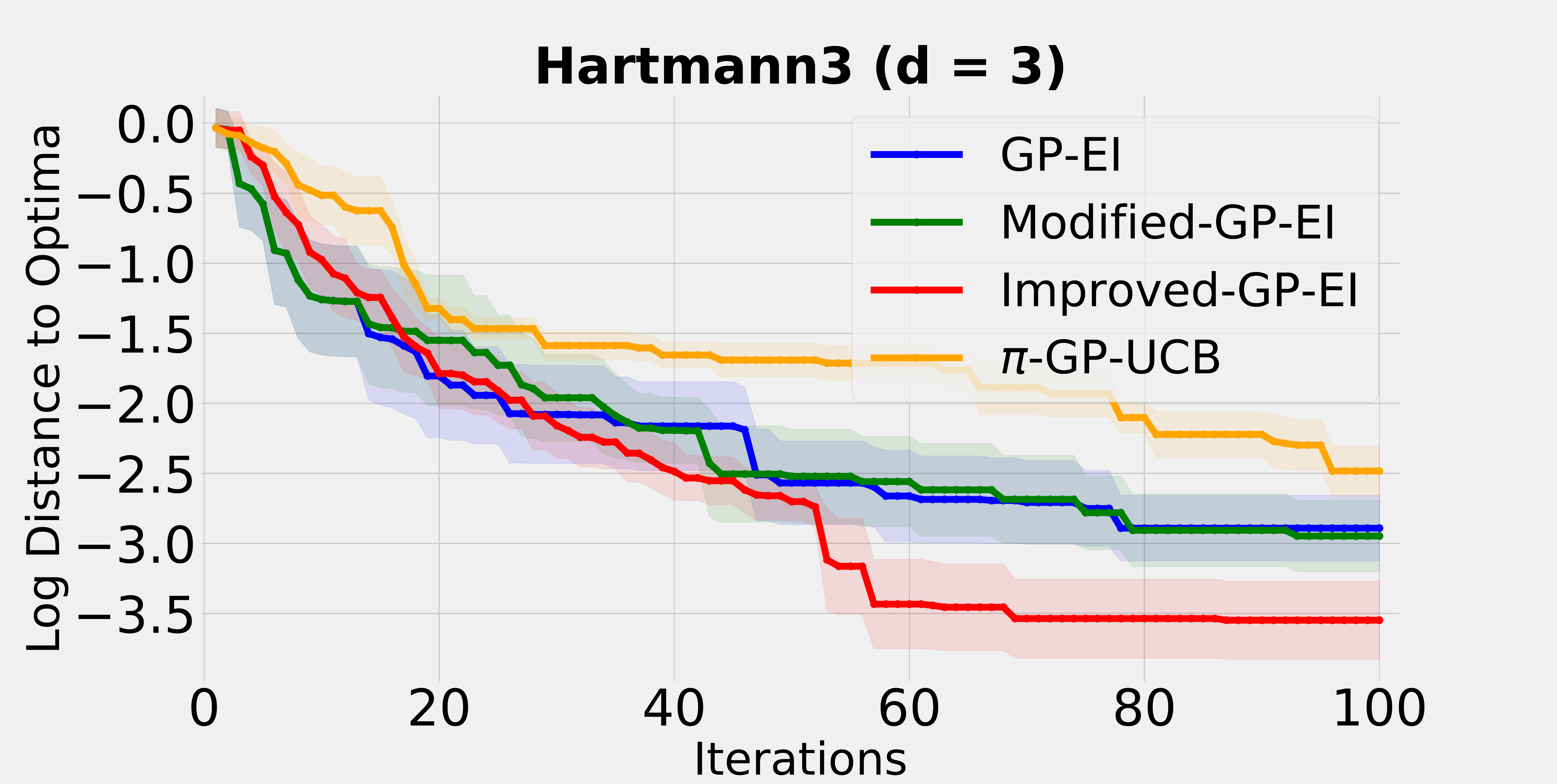}
}\hfill
\subfigure{\includegraphics[scale=1.0,width=.225\textwidth, height=.11\textheight]{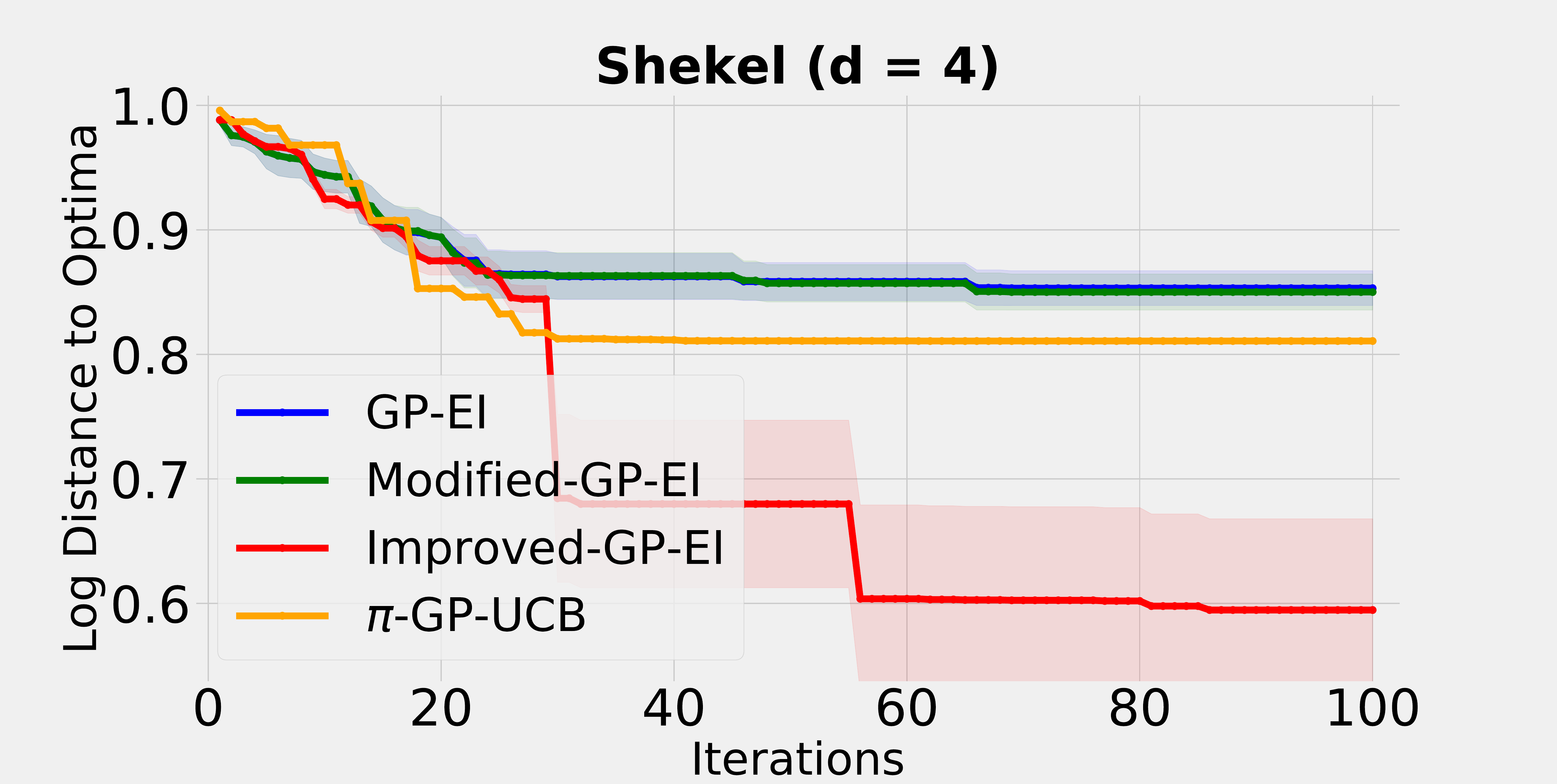}
}\hfill
\subfigure{\includegraphics[scale=1.0,width=.225\textwidth, height=.11\textheight]{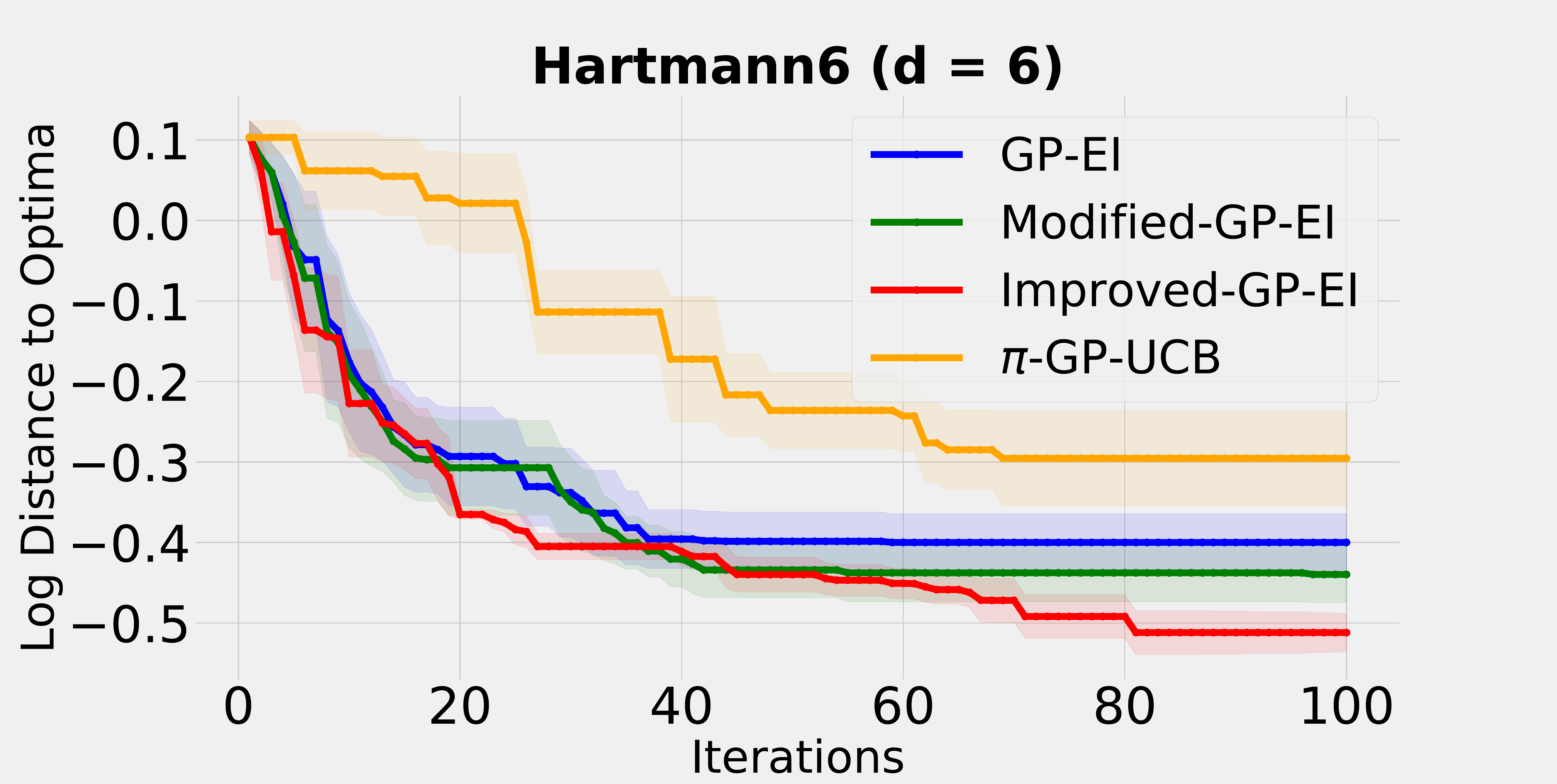}
}\hfill
\subfigure{\includegraphics[scale=1.0,width=.225\textwidth, height=.11\textheight]{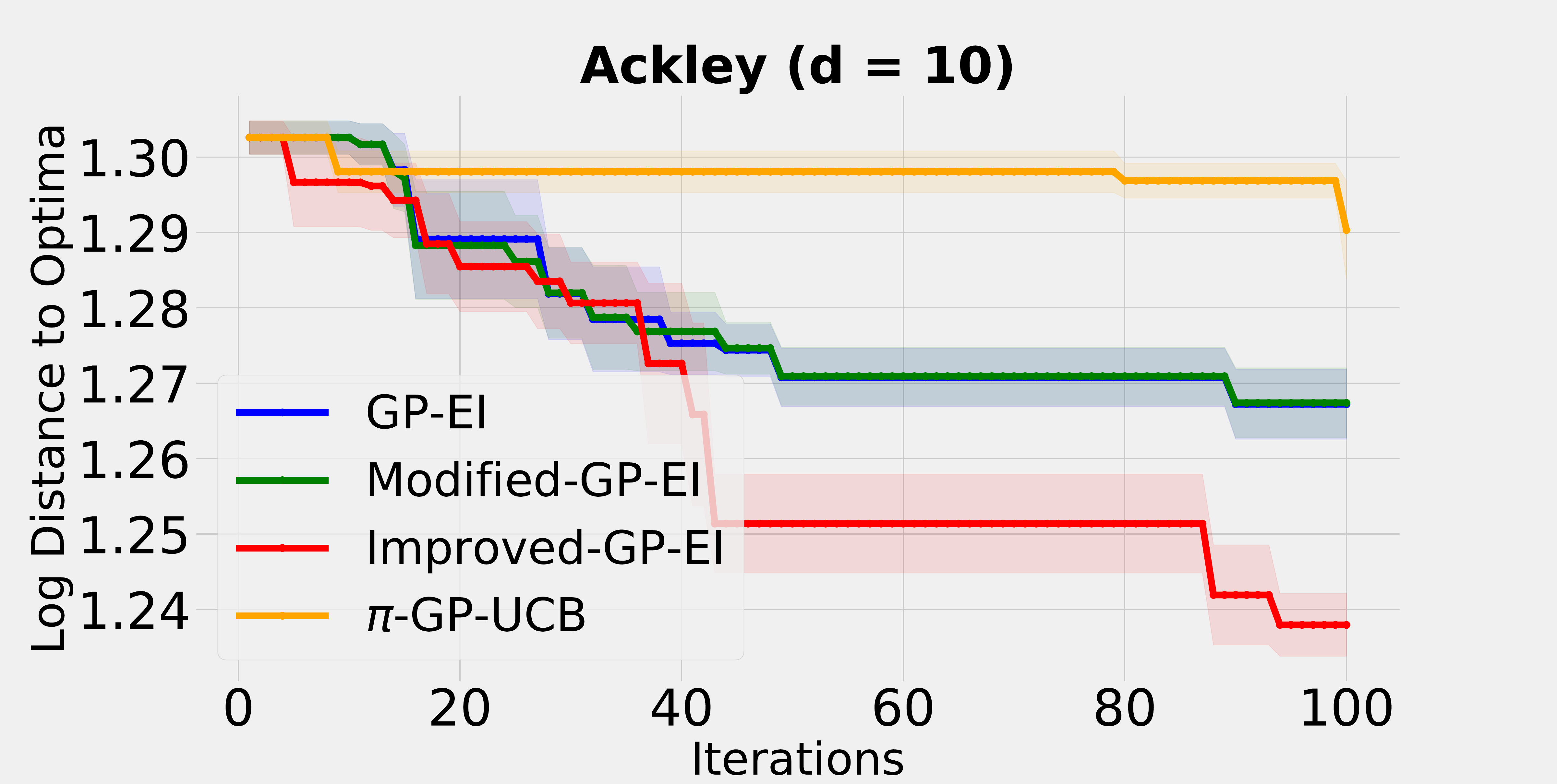}
}
\caption{Comparison of methods for Hartmann3, Shekel, Hartmann6 and Ackley functions.}
\label{fig1}
\end{figure*}
A recent review of GP-EI can be found in \citep{ZhanX20}.

In the noise-free setting, the most notable work in this sub-line is the one of \citet{Bull11} which showed that GP-EI can obtain an $\mathcal O(T^{-1/d})$ upper bound on the simple regret.

In noisy setting, \citet{Srinivas12} first introduced GP-UCB for both Bayesian and non-Bayesian settings. \citet{Valko13b} introduced KernelUCB for the case of a finite-armed bandit which can be extended to a continuum-armed bandit via a discretization argument. Another notable work for GP-UCB is Improved GP-UCB  \citep{Chowdhury17}, which offers a slightly faster convergence rate. A limitation of these algorithms is that the proposed regrets are sublinear only if $2\nu > d(d+1)$. To address this limitation, \citep{janz20a} recently introduced $\pi$-GP-UCB which is built upon Improved GP-UCB with guarantees that $\pi$ GP-UCB has a sublinear regret for every $\nu > 1$ and $d \ge 1$. GP-TS is another approach for Gaussian process bandit optimization. This was introduced by \citet{Chowdhury17} by extending the Thompson sampling algorithm in finite-armed bandits to continuum-armed bandits. \citet{scarlett17a} provided lower bounds for Gaussian process bandit optimization for both simple and cumulative regret. Recently \citet{vivarelli} provided new upper bounds for GP-UCB and GP-TS as well as the new bounds for the maximum information gain $\gamma_t$.

A limitation of GP-UCB and GP-TS algorithms is that the explorations are maintained via an upper confidence bound which requires to know several parameters e.g. the bound on the function RKHS norm, sub-Gaussianity level of the measurement noise. These parameters are usually unknown in practice. Consequently, these parameters are often set in a heuristic manner \citep{Berkenkamp17,Bogunovic18,janz20a,Wachi20}. Our GP-EI based algorithms avoid this limitation.

EI has also been studied in finite-armed bandit setting. \citet{Ryzhov16} studied EI for the problem of best-arm identification. Later, \citet{Qin17} proposed an improvement of this algorithm on computational efficiency. However, these results and analysis techniques do not apply to our settings of Gaussian processes and the RKHS norm.
\section{Experiments}
While the main focus of this paper is performing theoretical analysis for GP-EI in noiseless and noisy settings, we have also proposed a new algorithm termed as ``Improved-GP-EI" and also provided a variant of GP-EI which we will call Modified-GP-EI in noisy setting. In this section we have performed a comparison of our Improved-GP-EI as well as Modified-GP-EI against GP-EI with fixed $\omega$ and $\pi$-GP-UCB proposed by \citet{janz20a}, which has the tightest regret so far in the noisy setting. To compare the sample-efficiency of all the algorithms, we used four functions in noisy setting: Hartmann3 ($d =3$), Shekel ($d =4$), Hartmann6 ($d =6$) and Ackley ($d =10$). The evaluation metric is the log distance to the true optimum: $\text{log}_{10}(f(x^*) - f(x^+_t))$.

All implementations are in Python 3.6. For each test function, we repeat the experiments 15 times. We plot the mean and a confidence bound of one standard deviation across all the runs. We used Mat\'ern kernel with $\nu =2.5$ and the length scale $l =0.2$. Set $T = 100$, and $\delta = 0.05$. $\omega_T$ of Modified-GP-EI and Improved-GP-EI is set following Theorem \ref{theorem2} and Theorem \ref{theorem3}. For GP-EI, we use $\omega =1$ by default.

As seen from Figure \ref{fig1}, the Modified-GP-EI performs similar to GP-EI. This is expected as both methods only differ by $\omega_T$ , which is used for theoretical guarantee. $\pi$-GP-UCB is the most inefficient (except Shekel), probably because it requires to know the RKHS norm and sub-Gaussianity parameters which are unknown in practice. Following \citep{janz20a}, we used $B =1$, $R =1$. In contrast, Improved-GP-EI does not need to know such hyper-parameters and clearly outperforms $\pi$-GP-UCB.

We note that our Improved-GP-EI results in a significantly more scalable algorithm. The cover construction (i.e. partitioning the space) of Improved-GP-EI permits it to perform the inverse of the kernel $K_t$ on a subset of the data to reduce the computations from $O(t^3)$ to $O(\sum_{i=1}^{p}t^{3}_{i})$ where $t=\sum_{i=1}^{p}t_i$ and $p$ is the number of hypercubes in the partition. E.g. for Hartmann3, the average runtime per iteration of Improved-GP-EI is 1.21 mins compared to GP-EI's 1.37 min. This difference gets larger for larger horizons.
\subsection{Synthetic Test Functions} In this part, we benchmark on synthetic functions in RKHS spaces. The first function is built in the RKHS space equipped with a Mat\'ern kernel with $\nu = 2.5$ and $l = 0.2$. The second function is built in the RKHS space equipped with a SE kernel with $l = 1$.
We construct each function $f$ in a five-dimensional space by sampling 500 points $\hat{x}_1, .., \hat{x}_m$, uniformly on $[0,1]^5$, and $\hat{a}_1, .., \hat{a}_m$ each independent uniform on $[-1, 1]$ and defining $f(x) = \sum_{i=1}^{m} \hat{a}_i k(\hat{x}_i, x)$ for all $x \in \mathcal X$ and $k$ is a kernel. The RKHS norm of this function is computed as $||f||^2_k = \sum_{i,j=1}^{\infty} \hat{a}_i\hat{a}_jk(\hat{x}_j, \hat{x}_i)$. This norm is unknown for all the algorithms. Since the baseline $\pi$-GP-UCB requires to know this norm and the sub-Gaussianity parameter $R$, we set $B = 1$ and $R =1$ (in a heuristic manner) for $\pi$-GP-UCB. Otherwise, our proposed Improved-GP-EI algorithm  as well as Modified-GP-EI and GP-EI do not require to know these parameters. We used $\lambda = 0.01$ for the noise setting.
\begin{figure}[h]
\vspace{-5pt}
\centering
\subfigure{\includegraphics[scale=1.0,width=.225\textwidth,height=.11\textheight]{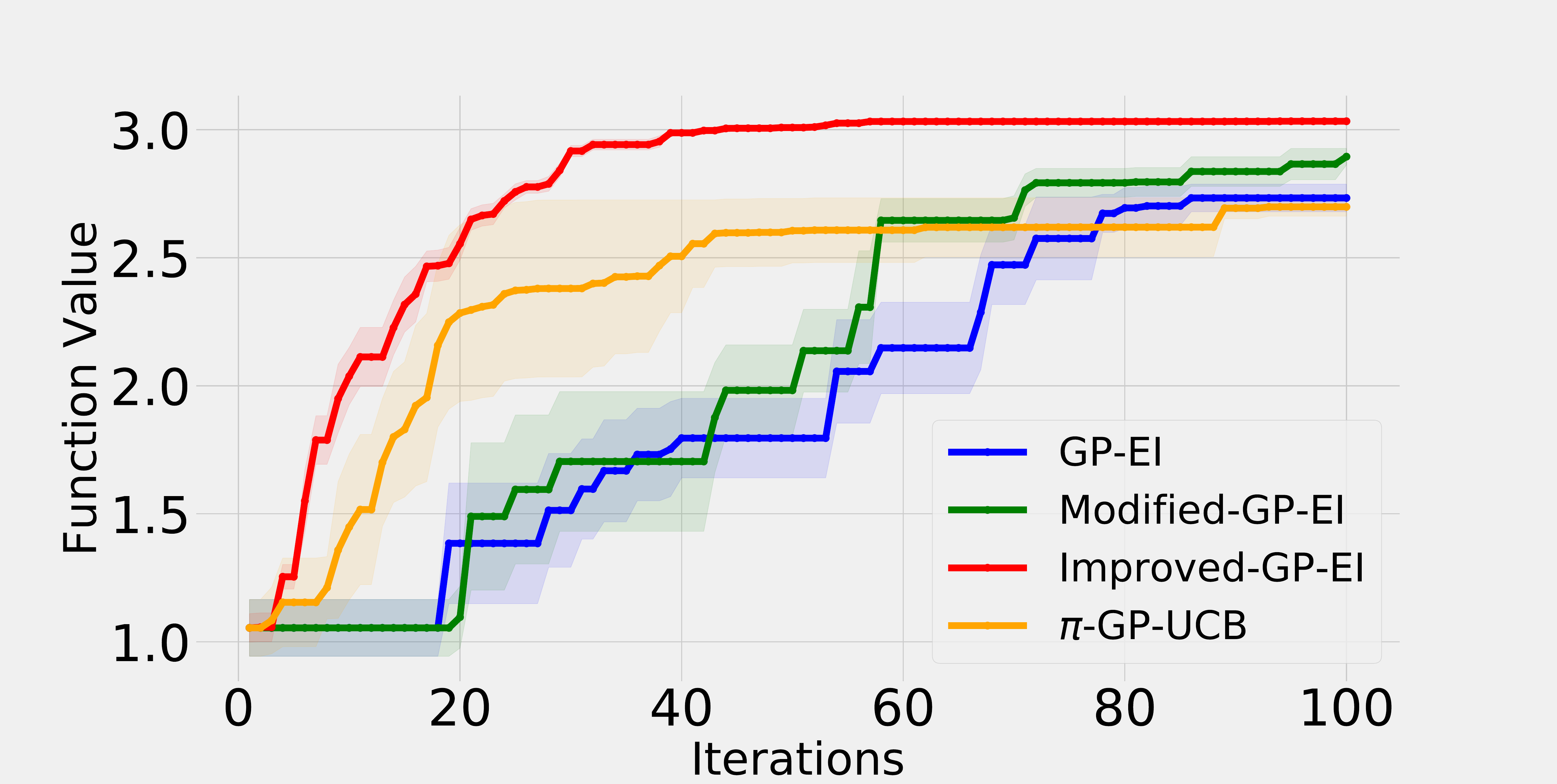}
}\hfill
\subfigure{\includegraphics[scale=1.0,width=.225\textwidth,height=.11\textheight]{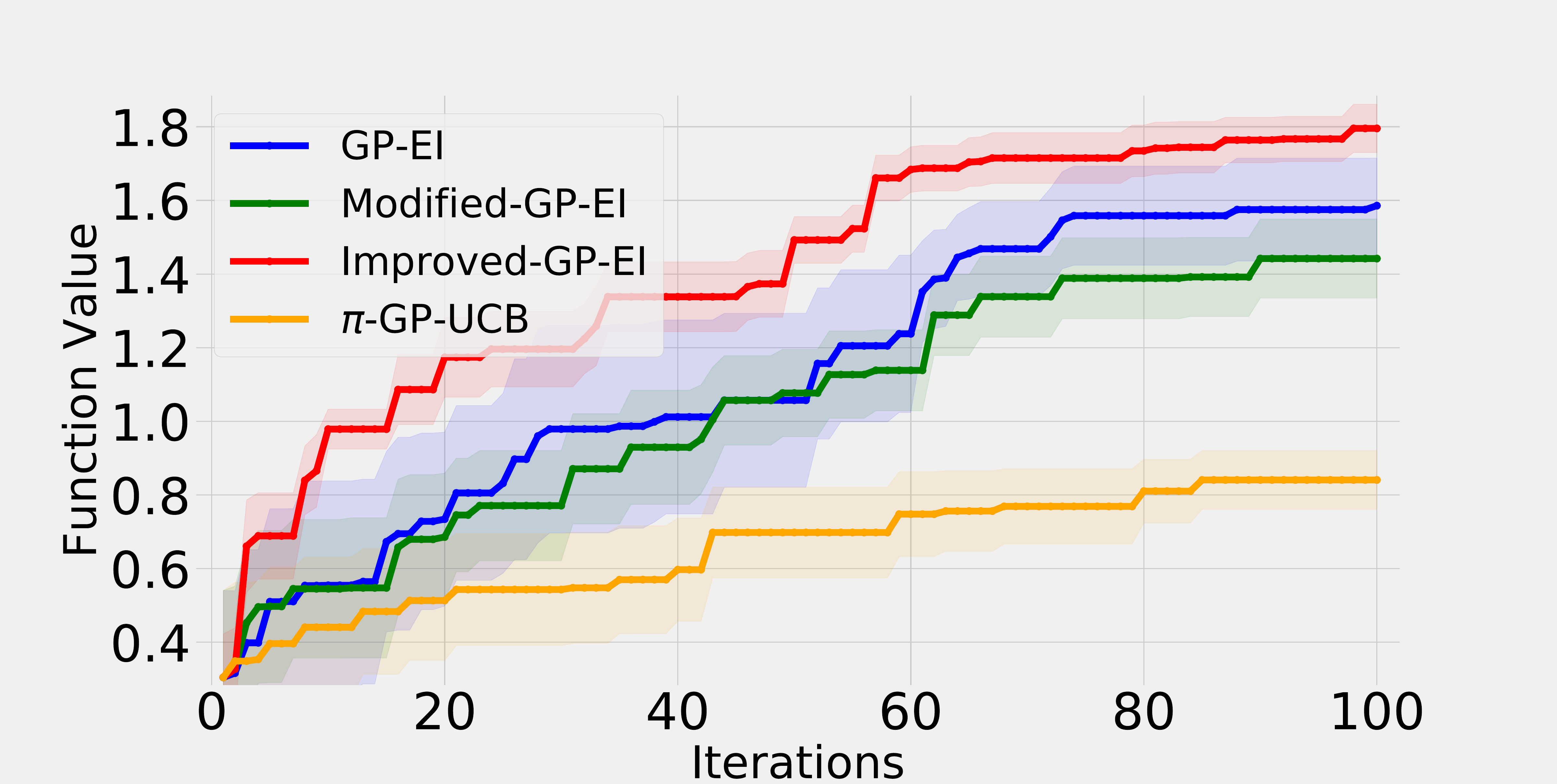}
}
\caption{Comparison of methods for the functions generated from the RKHS space. The left corresponds to the Mat\'ern kernel ($\nu = 2.5$, $l = 0.2$). The right corresponds to the Squared Exponential (SE) kernel ($l = 1.0$). The SE kernel is considered as a special case of the Mat\'ern kernel. We estimate the algorithms based on the function values at each iteration.}
\label{fig2}
\vspace{-10pt}
\end{figure}
As seen from Figure \ref{fig2}, for both the Mat\'ern kernel and the SE kernel, Improved-GP-EI outperforms all the other algorithms while
Modified-GP-EI and GP-EI show competitive performance. Improved-GP-EI performs better than Modified-GP-EI because it avoids using a large global scale variance parameter $\omega_T$. Improved-GP-EI uses $\omega_T$ growing only poly-logarithmically with
$T$. $\pi$-GP-UCB is outperformed by the EI based algorithms which are parameter-free. This is also observed in \citep{Chowdhury17}. We agree with \citet{Chowdhury17} that the UCB-based algorithms are somewhat less robust on the choice of kernel than EI-based algorithms.
\section{Conclusion}
We have demonstrated that GP-EI can converge with standard incumbent defined as the current best value of the GP predictive mean. Further we have proposed a variant of GP-EI, called Improved GP-EI which converges for every $\nu > 1$ and every $d \ge 1$. The empirical results have demonstrated the effectiveness of our proposed Improved GP-EI.

\subsubsection*{Acknowledgements}
This research was partially funded by the Australian Government through the Australian Research Council (ARC). Prof Venkatesh is the recipient of an ARC Australian Laureate Fellowship (FL170100006).

\bibliographystyle{plainnat}
\bibliography{Hung-research}

\begin{thebibliography}{37}
\providecommand{\natexlab}[1]{#1}
\providecommand{\url}[1]{\texttt{#1}}
\expandafter\ifx\csname urlstyle\endcsname\relax
  \providecommand{\doi}[1]{doi: #1}\else
  \providecommand{\doi}{doi: \begingroup \urlstyle{rm}\Url}\fi

\bibitem[Ryz(2016)]{Ryzhov16}
On the convergence rates of expected improvement methods.
\newblock \emph{Operations Research}, 64\penalty0 (6):\penalty0 1515--1528,
  2016.
\newblock Winner, INFORMS Simulation Society Outstanding Publication Award,
  2017.

\bibitem[Abbasi-yadkori et~al.(2011)Abbasi-yadkori, P\'{a}l, and
  Szepesv\'{a}ri]{Abbasi11}
Yasin Abbasi-yadkori, D\'{a}vid P\'{a}l, and Csaba Szepesv\'{a}ri.
\newblock Improved algorithms for linear stochastic bandits.
\newblock In J.~Shawe-Taylor, R.~Zemel, P.~Bartlett, F.~Pereira, and K.~Q.
  Weinberger, editors, \emph{Advances in Neural Information Processing
  Systems}, volume~24, pages 2312--2320. Curran Associates, Inc., 2011.
\newblock URL
  \url{https://proceedings.neurips.cc/paper/2011/file/e1d5be1c7f2f456670de3d53c7b54f4a-Paper.pdf}.

\bibitem[Agrawal and Goyal(2013)]{agrawal13}
Shipra Agrawal and Navin Goyal.
\newblock Thompson sampling for contextual bandits with linear payoffs.
\newblock In Sanjoy Dasgupta and David McAllester, editors, \emph{Proceedings
  of the 30th International Conference on Machine Learning}, volume~28 of
  \emph{Proceedings of Machine Learning Research}, pages 127--135, Atlanta,
  Georgia, USA, 17--19 Jun 2013. PMLR.

\bibitem[Balakrishnan et~al.(2020)Balakrishnan, Nguyen, Low, and
  Soh]{balakrishnan2020efficient}
Sreejith Balakrishnan, Quoc~Phong Nguyen, Bryan Kian~Hsiang Low, and Harold
  Soh.
\newblock Efficient exploration of reward functions in inverse reinforcement
  learning via bayesian optimization, 2020.

\bibitem[Bergstra et~al.(2011)Bergstra, Bardenet, Bengio, and
  K\'{e}gl]{Bergstra11a}
James Bergstra, R\'{e}mi Bardenet, Yoshua Bengio, and Bal\'{a}zs K\'{e}gl.
\newblock Algorithms for hyper-parameter optimization.
\newblock In J.~Shawe-Taylor, R.~Zemel, P.~Bartlett, F.~Pereira, and K.~Q.
  Weinberger, editors, \emph{Advances in Neural Information Processing
  Systems}, volume~24, pages 2546--2554. Curran Associates, Inc., 2011.
\newblock URL
  \url{https://proceedings.neurips.cc/paper/2011/file/86e8f7ab32cfd12577bc2619bc635690-Paper.pdf}.

\bibitem[Berkenkamp et~al.(2017)Berkenkamp, Turchetta, Schoellig, and
  Krause]{Berkenkamp17}
Felix Berkenkamp, Matteo Turchetta, Angela Schoellig, and Andreas Krause.
\newblock Safe model-based reinforcement learning with stability guarantees.
\newblock In I.~Guyon, U.~V. Luxburg, S.~Bengio, H.~Wallach, R.~Fergus,
  S.~Vishwanathan, and R.~Garnett, editors, \emph{Advances in Neural
  Information Processing Systems}, volume~30, pages 908--918. Curran
  Associates, Inc., 2017.
\newblock URL
  \url{https://proceedings.neurips.cc/paper/2017/file/766ebcd59621e305170616ba3d3dac32-Paper.pdf}.

\bibitem[Berkenkamp et~al.(2019)Berkenkamp, Schoellig, and
  Krause]{Berkenkamp19}
Felix Berkenkamp, Angela~P. Schoellig, and Andreas Krause.
\newblock No-regret bayesian optimization with unknown hyperparameters.
\newblock \emph{Journal of Machine Learning Research}, 20\penalty0
  (50):\penalty0 1--24, 2019.

\bibitem[Bogunovic et~al.(2018)Bogunovic, Scarlett, Jegelka, and
  Cevher]{Bogunovic18}
Ilija Bogunovic, Jonathan Scarlett, Stefanie Jegelka, and Volkan Cevher.
\newblock Adversarially robust optimization with gaussian processes.
\newblock NIPS'18, page 5765–5775, Red Hook, NY, USA, 2018. Curran Associates
  Inc.

\bibitem[Bull(2011)]{Bull11}
Adam~D. Bull.
\newblock Convergence rates of efficient global optimization algorithms.
\newblock \emph{J. Mach. Learn. Res.}, 12:\penalty0 2879--2904, November 2011.
\newblock ISSN 1532-4435.

\bibitem[Chowdhury et~al.(2017)Chowdhury, Gopalan, and abc]{Chowdhury17}
Sayak~Ray Chowdhury, Aditya Gopalan, and abc.
\newblock On kernelized multi-armed bandits.
\newblock ICML'17, page 844–853. JMLR.org, 2017.

\bibitem[De~Freitas et~al.(2012)De~Freitas, Smola, and Zoghi]{Freitas12}
Nando De~Freitas, Alex~J. Smola, and Masrour Zoghi.
\newblock Exponential regret bounds for gaussian process bandits with
  deterministic observations.
\newblock In \emph{Proceedings of the 29th International Coference on
  International Conference on Machine Learning}, ICML’12, page 955–962,
  Madison, WI, USA, 2012. Omnipress.
\newblock ISBN 9781450312851.

\bibitem[Hern\'{a}ndez-Lobato et~al.(2014)Hern\'{a}ndez-Lobato, Hoffman, and
  Ghahramani]{Lobato14}
Jos\'{e}~Miguel Hern\'{a}ndez-Lobato, Matthew~W Hoffman, and Zoubin Ghahramani.
\newblock Predictive entropy search for efficient global optimization of
  black-box functions.
\newblock In Z.~Ghahramani, M.~Welling, C.~Cortes, N.~Lawrence, and K.~Q.
  Weinberger, editors, \emph{Advances in Neural Information Processing
  Systems}, volume~27. Curran Associates, Inc., 2014.

\bibitem[Hoffman et~al.(2014)Hoffman, Shahriari, and Freitas]{hoffman14}
Matthew Hoffman, Bobak Shahriari, and Nando Freitas.
\newblock {On correlation and budget constraints in model-based bandit
  optimization with application to automatic machine learning}.
\newblock In Samuel Kaski and Jukka Corander, editors, \emph{Proceedings of the
  Seventeenth International Conference on Artificial Intelligence and
  Statistics}, volume~33 of \emph{Proceedings of Machine Learning Research},
  pages 365--374, Reykjavik, Iceland, 22--25 Apr 2014. PMLR.
\newblock URL \url{http://proceedings.mlr.press/v33/hoffman14.html}.

\bibitem[Janz et~al.(2020)Janz, Burt, and Gonzalez]{janz20a}
David Janz, David Burt, and Javier Gonzalez.
\newblock Bandit optimisation of functions in the matérn kernel rkhs.
\newblock In Silvia Chiappa and Roberto Calandra, editors, \emph{Proceedings of
  the Twenty Third International Conference on Artificial Intelligence and
  Statistics}, volume 108 of \emph{Proceedings of Machine Learning Research},
  pages 2486--2495. PMLR, 26--28 Aug 2020.

\bibitem[Lizotte et~al.(2007)Lizotte, Wang, Bowling, and Schuurmans]{Lizotte07}
Daniel Lizotte, Tao Wang, Michael Bowling, and Dale Schuurmans.
\newblock Automatic gait optimization with gaussian process regression.
\newblock IJCAI'07, page 944–949, San Francisco, CA, USA, 2007. Morgan
  Kaufmann Publishers Inc.

\bibitem[Malkomes and Garnett(2018)]{Malkomes2018}
Gustavo Malkomes and Roman Garnett.
\newblock Automating bayesian optimization with bayesian optimization.
\newblock In S.~Bengio, H.~Wallach, H.~Larochelle, K.~Grauman, N.~Cesa-Bianchi,
  and R.~Garnett, editors, \emph{Advances in Neural Information Processing
  Systems}, volume~31. Curran Associates, Inc., 2018.
\newblock URL
  \url{https://proceedings.neurips.cc/paper/2018/file/2b64c2f19d868305aa8bbc2d72902cc5-Paper.pdf}.

\bibitem[{Marchant} and {Ramos}(2012)]{Marchan12}
R.~{Marchant} and F.~{Ramos}.
\newblock Bayesian optimisation for intelligent environmental monitoring.
\newblock In \emph{2012 IEEE/RSJ International Conference on Intelligent Robots
  and Systems}, pages 2242--2249, 2012.
\newblock \doi{10.1109/IROS.2012.6385653}.

\bibitem[Martinez-cantin et~al.(2007)Martinez-cantin, de~Freitas, Doucet, and
  Castellanos]{cantin07}
Ruben Martinez-cantin, Nando de~Freitas, Arnaud Doucet, and José~A.
  Castellanos.
\newblock Active policy learning for robot planning and exploration under
  uncertainty.
\newblock In \emph{IN PROCEEDINGS OF ROBOTICS: SCIENCE AND SYSTEMS}, 2007.

\bibitem[Mo{\v{c}}kus(1975)]{Mockus74}
J.~Mo{\v{c}}kus.
\newblock On bayesian methods for seeking the extremum.
\newblock In G.~I. Marchuk, editor, \emph{Optimization Techniques IFIP
  Technical Conference Novosibirsk, July 1--7, 1974}, pages 400--404, Berlin,
  Heidelberg, 1975. Springer Berlin Heidelberg.
\newblock ISBN 978-3-540-37497-8.

\bibitem[Nguyen and Osborne(2020)]{nguyen20d}
Vu~Nguyen and Michael~A. Osborne.
\newblock Knowing the what but not the where in {B}ayesian optimization.
\newblock In Hal~Daumé III and Aarti Singh, editors, \emph{Proceedings of the
  37th International Conference on Machine Learning}, volume 119 of
  \emph{Proceedings of Machine Learning Research}, pages 7317--7326. PMLR,
  13--18 Jul 2020.

\bibitem[Nguyen et~al.(2017)Nguyen, Gupta, Rana, Li, and Venkatesh]{nguyen17a}
Vu~Nguyen, Sunil Gupta, Santu Rana, Cheng Li, and Svetha Venkatesh.
\newblock Regret for expected improvement over the best-observed value and
  stopping condition.
\newblock In Min-Ling Zhang and Yung-Kyun Noh, editors, \emph{Proceedings of
  the Ninth Asian Conference on Machine Learning}, volume~77 of
  \emph{Proceedings of Machine Learning Research}, pages 279--294. PMLR, 15--17
  Nov 2017.

\bibitem[Osborne(2010)]{osborne2010}
Michael~A. Osborne.
\newblock Bayesian gaussian processes for sequential prediction, optimisation
  and quadrature.
\newblock 2010.

\bibitem[Qin et~al.(2017)Qin, Klabjan, and Russo]{Qin17}
Chao Qin, Diego Klabjan, and Daniel Russo.
\newblock Improving the expected improvement algorithm.
\newblock In \emph{Proceedings of the 31st International Conference on Neural
  Information Processing Systems}, NIPS'17, page 5387–5397, Red Hook, NY,
  USA, 2017. Curran Associates Inc.
\newblock ISBN 9781510860964.

\bibitem[Rasmussen and Williams(2005)]{Rasmussen05}
Carl~Edward Rasmussen and Christopher K.~I. Williams.
\newblock \emph{Gaussian Processes for Machine Learning (Adaptive Computation
  and Machine Learning)}.
\newblock The MIT Press, 2005.
\newblock ISBN 026218253X.

\bibitem[Scarlett et~al.(2017)Scarlett, Bogunovic, and Cevher]{scarlett17a}
Jonathan Scarlett, Ilija Bogunovic, and Volkan Cevher.
\newblock Lower bounds on regret for noisy {G}aussian process bandit
  optimization.
\newblock In Satyen Kale and Ohad Shamir, editors, \emph{Proceedings of the
  2017 Conference on Learning Theory}, volume~65 of \emph{Proceedings of
  Machine Learning Research}, pages 1723--1742, Amsterdam, Netherlands, 07--10
  Jul 2017. PMLR.

\bibitem[Snoek et~al.(2012)Snoek, Larochelle, and Adams]{Snoek12}
Jasper Snoek, Hugo Larochelle, and Ryan~P. Adams.
\newblock Practical bayesian optimization of machine learning algorithms.
\newblock In \emph{Proceedings of the 25th International Conference on Neural
  Information Processing Systems - Volume 2}, NIPS'12, page 2951–2959, Red
  Hook, NY, USA, 2012. Curran Associates Inc.

\bibitem[Srinivas et~al.(2012)Srinivas, Krause, Kakade, and Seeger]{Srinivas12}
Niranjan Srinivas, Andreas Krause, Sham~M. Kakade, and Matthias~W. Seeger.
\newblock Information-theoretic regret bounds for gaussian process optimization
  in the bandit setting.
\newblock \emph{IEEE Trans. Inf. Theor.}, 58\penalty0 (5):\penalty0 3250--3265,
  May 2012.
\newblock ISSN 0018-9448.
\newblock \doi{10.1109/TIT.2011.2182033}.
\newblock URL \url{http://dx.doi.org/10.1109/TIT.2011.2182033}.

\bibitem[Stein(1999)]{stein1999}
Michael~L. Stein.
\newblock \emph{Interpolation of spatial data}.
\newblock Springer Series in Statistics. Springer-Verlag, New York, 1999.
\newblock ISBN 0-387-98629-4.
\newblock \doi{10.1007/978-1-4612-1494-6}.
\newblock URL \url{http://dx.doi.org/10.1007/978-1-4612-1494-6}.
\newblock Some theory for Kriging.

\bibitem[Tran-The et~al.(2020)Tran-The, Gupta, Rana, Ha, and
  Venkatesh]{Tran-The_20}
Hung Tran-The, Sunil Gupta, Santu Rana, Huong Ha, and Svetha Venkatesh.
\newblock Sub-linear regret bounds for bayesian optimisation in unknown search
  spaces.
\newblock In H.~Larochelle, M.~Ranzato, R.~Hadsell, M.~F. Balcan, and H.~Lin,
  editors, \emph{Advances in Neural Information Processing Systems}, volume~33,
  pages 16271--16281. Curran Associates, Inc., 2020.
\newblock URL
  \url{https://proceedings.neurips.cc/paper/2020/file/bb073f2855d769be5bf191f6378f7150-Paper.pdf}.

\bibitem[Tran-The et~al.(2021)Tran-The, Gupta, Rana, and
  Venkatesh]{tran-the21a}
Hung Tran-The, Sunil Gupta, Santu Rana, and Svetha Venkatesh.
\newblock Bayesian optimistic optimisation with exponentially decaying regret.
\newblock In Marina Meila and Tong Zhang, editors, \emph{Proceedings of the
  38th International Conference on Machine Learning}, volume 139 of
  \emph{Proceedings of Machine Learning Research}, pages 10390--10400. PMLR,
  18--24 Jul 2021.

\bibitem[Vakili et~al.(2021)Vakili, Khezeli, and Picheny]{vakili21a}
Sattar Vakili, Kia Khezeli, and Victor Picheny.
\newblock On information gain and regret bounds in gaussian process bandits.
\newblock In Arindam Banerjee and Kenji Fukumizu, editors, \emph{Proceedings of
  The 24th International Conference on Artificial Intelligence and Statistics},
  volume 130 of \emph{Proceedings of Machine Learning Research}, pages 82--90.
  PMLR, 13--15 Apr 2021.

\bibitem[Valko et~al.(2013)Valko, Korda, Munos, Flaounas, and
  Cristianini]{Valko13b}
Michal Valko, Nathan Korda, R\'{e}mi Munos, Ilias Flaounas, and Nello
  Cristianini.
\newblock Finite-time analysis of kernelised contextual bandits.
\newblock In \emph{Proceedings of the Twenty-Ninth Conference on Uncertainty in
  Artificial Intelligence}, UAI'13, page 654–663, Arlington, Virginia, USA,
  2013. AUAI Press.

\bibitem[Vivarelli(1998)]{vivarelli}
Francesco Vivarelli.
\newblock Studies on the generalisation of gaussian processes and bayesian
  neural networks.
\newblock In \emph{PhD thesis}. Aston University, 1998.

\bibitem[Wachi and Sui(2020)]{Wachi20}
Akifumi Wachi and Yanan Sui.
\newblock Safe reinforcement learning in constrained {M}arkov decision
  processes.
\newblock In Hal~Daumé III and Aarti Singh, editors, \emph{Proceedings of the
  37th International Conference on Machine Learning}, volume 119 of
  \emph{Proceedings of Machine Learning Research}, pages 9797--9806, Virtual,
  13--18 Jul 2020. PMLR.

\bibitem[Wang and de~Freitas(2014)]{wang2014}
Ziyu Wang and Nando de~Freitas.
\newblock Theoretical analysis of bayesian optimisation with unknown gaussian
  process hyper-parameters, 2014.

\bibitem[Wilson et~al.(2014)Wilson, Fern, and Tadepalli]{wilson14a}
Aaron Wilson, Alan Fern, and Prasad Tadepalli.
\newblock Using trajectory data to improve bayesian optimization for
  reinforcement learning.
\newblock \emph{Journal of Machine Learning Research}, 15\penalty0
  (8):\penalty0 253--282, 2014.
\newblock URL \url{http://jmlr.org/papers/v15/wilson14a.html}.

\bibitem[Zhan and Xing(2020)]{ZhanX20}
Dawei Zhan and Huanlai Xing.
\newblock Expected improvement for expensive optimization: a review.
\newblock \emph{J. Glob. Optim.}, 78\penalty0 (3):\penalty0 507--544, 2020.
\newblock \doi{10.1007/s10898-020-00923-x}.
\newblock URL \url{https://doi.org/10.1007/s10898-020-00923-x}.

\end{thebibliography}


\clearpage
\appendix

\thispagestyle{empty}

\onecolumn \makesupplementtitle

Before providing the theoretical results and the additional experiments, we summarize some important notations used in our proofs in Table 2 and in section A we first explain the derivation of the EI acquisition function in Section 2.3. Next, in Section B and C we provide the proof for Theorem 1 and Theorem 2 in the main paper, respectively.
\begin{table}[H]
\centering
\caption{A summary of the notations used in the paper.}
\scalebox{0.9}{
\begin{tabular}[t]{|l|l|}\hline
\textbf{Common}  & \textbf{Definition} \\ \hline \hline
$d$ & the number of dimensions of the search space $\mathcal X$ \\
$x^*$ & $\text{argmax}_{x \in \mathcal X} f(x)$, an optimal  \\
$B$ & the upper bound of the RKHS norm \\
$R$ & the sub-Gaussianity parameter of the noise \\
$\nu$ & the parameter controlling the smoothness of the function\\
$l$ & the lengthscale of the kernel \\
$\alpha_t^{EI}(x)$ & the EI acquisition function at iteration $t$ \\
$\xi$ & the incumbent of EI \\
$\Phi(z)$ & the standard normal distribution function \\
$\phi(z)$ & the density function\\
$\tau(z)$ & $z\Phi(z) + \phi(z)$ which is a increasing function \\ \hline
\textbf{Section B}  &  \\ \hline \hline
$\omega_t$ &  the replacement of $\omega$ in the noisy setting\\
$x_t^+$ &  $x^+_t = \text{argmax} \{\mu_{t-1}(x_i)\}_{x_i \in \mathcal D_{t-1}}$ \\
$\mu_t^+$ & the best GP predictive mean, $\mu^+_t = \text{max} \{\mu_{t-1}(x_i)\}_{x_i \in \mathcal D_t}$ \\
$\gamma_t$ & the maximum information gain up to $t$ iterations \\ \hline
\textbf{Section C}  &  \\ \hline \hline
$b$ & the constant defined as $b = \frac{d+1}{d + 2\nu}$ \\
$q$ & the constant defined as $q = \frac{d(d+1)}{d(d+2) + 2\nu}$ \\
$\rho_A$  & the diameter of the hypercube $A$ \\
$\mathcal D_t^A$ & the subset of $\mathcal D_t$ in $A$ \\
$K_t^A$ & the kernel matrix, defined as $K_t^A = [k(x, x')]_{x, x' \in \mathcal D_t^A}$ \\
$\gamma_t^A$ & the information gain for $A$, defined as $\gamma_t^A = \frac{1}{2}\text{log}|I + \lambda^{-1}K_t^A|$ (See \citep{janz20a})\\
$\mathcal A_t$ & the set of the newly created hypercubes at iteration $t$ and the hypercubes of $\mathcal A_{t-1}$ that were not split \\ \hline
\end{tabular}}
\label{Table_Time}
\end{table}%
\section{Derivation of the Expected Improvement in Section 2.3.}
Set $I_t(x) = \text{max}\{0, f(x) - \mu^+_t\}$. $I_t(x)$ is positive when the prediction is higher than the best value known
thus far. Otherwise, $I_t(x)$ is set to zero. The new query point is found by
maximizing the expected improvement:
$$x = \text{argmax}_{x} \mathbb{E}(I_t(x)).$$
As defined in Section 3.2, for every $x \in \mathcal X$, the posterior distribution of $f(x)$, at iteration $t$, is $\mathcal N(\mu_{t-1}(x), \omega^2\sigma^2_{t-1}(x))$. Therefore, the likelihood of improvement $I_t$ on a normal posterior distribution characterized
by $\mu_{t-1}(x), \omega\sigma_{t-1}(x)$ can be computed from the normal density function:
$$\frac{1}{\sqrt{2\pi}\omega\sigma_t(x)}\text{exp}(-\frac{(\mu_t(x) - f(x)- I_t(x))^2}{2\omega^2\sigma^2_t(x)}).$$
The expected improvement is the integral over this function:
\begin{eqnarray*}
\mathbb{E}(I_t) & = & \int_{I_t = 0}^{I_t = \infty} I_t\frac{1}{\sqrt{2\pi}\omega\sigma_t(x)}\text{exp}(-\frac{(\mu_{t-1}(x) - f(x)- I_t(x))^2}{2\omega^2\sigma^2_t(x)})\;\mathrm{d}t \\
& = & \omega\sigma_t(x)[\frac{\mu_{t-1}(x) - \mu^+_t}{\omega\sigma_{t-1}(x)} \Phi(\frac{\mu_{t-1}(x) - \mu^+_t}{\omega\sigma_{t-1}(x)}) + \phi(\frac{\mu_{t-1}(x) - \mu^+_t}{\omega\sigma_{t-1}(x)})] \\
& = & (\mu_{t-1}(x) - \mu^+_t)\Phi(\frac{\mu_{t-1}(x) - \mu^+_t}{\omega\sigma_{t-1}(x)}) + \omega\sigma_{t-1}(x)\phi(\frac{\mu_{t-1}(x) - \mu^+_t}{\omega\sigma_{t-1}(x)})
\end{eqnarray*}
Setting $u = \mu_{t-1}(x) - \mu^+_t$ and $v = \omega\sigma_{t-1}(x)$, and $\rho(u, v) = u\Phi(\frac{u}{v}) + v\phi(\frac{u}{v})$, we obtain the formula of EI as : $$\alpha_t^{EI}(x) = \rho(u,v) = \rho(\mu_{t-1}(x) - \mu^+_t,\omega\sigma_{t-1}(x)).$$
\section{Proof of Theorem 1}
Instead of upper bounding directly the simple regret, we will seek to upper bound the sum $\sum_{t=1}^{T} r_t$ to exploit the results from the maximum information gain ($\gamma_T$) in the noisy setting. The proof for Theorem 1 involves two steps.
\begin{itemize}
  \item Upper bounding the instantaneous regret $r_t = f(x^*) - f(x_t^+)$.
  \item Upper bounding the sum $\sum_{t=1}^{T} r_t$
\end{itemize}
\subsection{Upper bounding the instantaneous regret $r_t = f(x^*) - f(x_t^+)$:} To obtain a bound on $r_t$, we break down $r_t$ into two terms as follows:
\begin{eqnarray*}
r_t & = & f(x^*) - f(x_t^+) \\
& = &  \underbrace{f(x^*) - \mu^+_{t}}_{\text{Term 1}} + \underbrace{\mu_{t}^+ - f(x_t^+)}_{\text{Term 2}}
\end{eqnarray*}
\paragraph{Upper Bounding Term 1.}
First, we provide the lower bound and upper bound for the acquisition function $\alpha^{EI}_t(x$ in the noisy setting. These bounds are similar to those in the noiseless setting (See Lemma 4).
\begin{lemma}[Based on Lemma 9 of \citep{wang2014}]
Pick $\delta \in (0,1)$. For $x \in \mathcal X$, $t \in \mathbb{N}$, set $I_t(x) = \text{max}\{0, f(x)- \mu^+_t\}$. Then with  probability at least $1 -\delta$ we have
$$I_t(x) - \beta_t\sigma_{t-1}(x) \le \alpha^{EI}_t(x) \le I_t(x) + (\beta_t + \omega_t)\sigma_{t-1}(x).$$ \label{lem2.1}
\end{lemma}
\begin{proof*}
If $\sigma_{t-1}(x) = 0$ then $\alpha_t^{EI}(x) = I_t(x)$, which makes the result trivial. We now assume that $\sigma_{t-1}(x) > 0$. Set $q = \frac{f(x) - \mu^+_t}{\sigma_{t-1}(x)}$ and $u = \frac{\mu_{t-1}(x) - \mu^+_{t}}{\sigma_{t-1}(x)}$. Then we have that
$$\alpha_t^{EI}(x) = \omega_t \sigma_{t-1}(x) \tau(\frac{u}{\omega_t}).$$
By Lemma \ref{le_beta_2}, we have that $|u - q| \le \beta_t$ with probability $1 -\delta$. Set $\tau(z) = z\Phi(z) + \phi(z)$. As $\tau'(z) = \Phi(z) \in [0,1]$, $\tau$ is non-decreasing and $\tau(z) \le 1 + z$ for $z > 0$. Hence,
\begin{eqnarray*}
\alpha_t^{EI}(x) & \le & \omega_t\sigma_{t-1}(x)\tau(\frac{\text{max}\{0, q\} + \beta_t}{\omega_t}) \\
& \le & \omega_t \sigma_{t-1}(x)(\frac{\text{max}\{0, q\} + \beta_t}{\omega_t} + 1) \\
& = & I_t(x) + (\beta_t + \omega_t) \sigma_{t-1}(x)
\end{eqnarray*}
If $I_t(x) = 0$ then the lower bound is trivial as $\alpha_t^{EI}(x)$ is non-negative. Thus suppose $I_t(x) > 0$. Since $\alpha_t^{EI}(x) \ge 0$ and $\tau(z) \ge 0$ for all $z$, and $\tau(z) = z + \tau(-z) \ge z$. Therefore,
\begin{eqnarray*}
\alpha_t^{EI}(x) & \ge & \omega_t\sigma_{t-1}(x)\tau(\frac{q -\beta_t}{\omega_t}) \\
& \ge & \omega_t \sigma_{t-1}(x)(\frac{q -\beta_t}{\omega_t}) \\
& = & I_t(x) -  \beta_t \sigma_{t-1}(x)
\end{eqnarray*}
\qedhere
\end{proof*}
Now we will use the results from Lemma \ref{lem2.1} to upper bound Term 1 as in the follow lemma.
\begin{lemma}
Pick $\delta \in (0,0.5)$. Then with probability at least $1 -2\delta$ we have
$$ f(x^*) - \mu_{t}^+ \le  \frac{\tau(\frac{\beta_t}{\omega_t})}{ \tau(-\frac{\beta_t}{\omega_t})}( \text{max}\{0, f(x_{t}) - \mu_{t-1}^+\} + (\beta_t + \omega_t)\sigma_{t-1}(x_{t})).$$ \label{lem2.2}
\end{lemma}
\begin{proof*}
If $\sigma_{t-1}(x^*) = 0$ then by definition of $\alpha_t^{EI} (x)$, we have  $\alpha_t^{EI} (x^*) = I_t(x^*)$.  We have
\begin{eqnarray*}
I_t(x^*) & = & \alpha^{EI}_t(x^*) \\
& \le & \alpha^{EI}_t(x_{t}) \\
& \le &  \text{max} \{0, f(x_{t})- \mu^+_{t}\} + (\beta_t + \omega_t)\sigma_{t-1}(x_{t}) \\
& \le & \frac{\tau(\frac{\beta_t}{\omega_t})}{ \tau(-\frac{\beta_t}{\omega_t})}( \text{max}\{0, f(x_{t}) - \mu_{t-1}^+\} + (\beta_t + \omega_t)\sigma_{t-1}(x_{t})),
\end{eqnarray*}
where in the first inequality, we use the definition $\alpha^{EI}_t(x_{t}) = \text{max}_{x \in \mathcal X} \alpha^{EI}_t(x)$. In the second inequality, we use Lemma \ref{lem2.1}. The third inequality holds since $\frac{\tau(\frac{\beta_t}{\omega_t})}{ \tau(-\frac{\beta_t}{\omega_t})} \ge \frac{\tau(0)}{\tau(0)} = 1$ due to the fact that function $\tau(z)$ is an increasing function. Thus, the lemma holds with probability $1 - \delta$.

We now consider $\sigma_{t-1}(x^*) >0$. If $f(x^*) < \mu_{t}^+$ then the lemma will be trivial. We now consider $f(x^*) \ge \mu_{t}^+$. By Lemma \ref{le_beta_2}, $\mu_{t-1}(x^*) - f(x^*) \ge -\beta_t \sigma_{t-1}(x^*)$ with probability $ 1 -\delta$. Combining with the fact that $f(x^*) \ge \mu_{t}^+$, we have that $\frac{\mu_{t-1}(x^*) - \mu^+_{t}}{\omega_t  \sigma_{t-1}(x^*)} \ge \frac{-\beta_t}{\omega_t}$  with probability $ 1 -\delta$. On the other hand, following the derivation of the acquisition function $EI$, we have $\alpha_t^{EI}(x^*) = \omega_t  \sigma_{t-1}(x^*) \tau(\frac{\mu_{t-1}(x^*) - \mu^+_{t}}{\omega_t  \sigma_{t-1}(x^*)})$. Therefore, $\alpha_t^{EI}(x^*) \ge \omega_t \sigma_{t-1}(x^*) \tau(-\beta_t/\omega_t)$ with probability $ 1 -\delta$ due to the fact that $\tau(z)$ is an increasing function.

By combining inequalities $\alpha_t^{EI}(x^*) \ge \omega_t \sigma_{t-1}(x^*) \tau(-\frac{\beta_t}{\omega_t})$, $\alpha_t^{EI}(x^*) \ge I_t(x^*) -  \beta_t \sigma_{t-1}(x^*)$ which is proven in Lemma \ref{lem2.1}, we obtain $ \frac{\beta_t}{(\omega_t\tau(-\frac{\beta_t}{\omega_t}))}\alpha_t^{EI}(x^*) + \alpha_t^{EI}(x^*) \ge I_t(x^*)$. Now,using the fact that $\tau(z) = z + \tau(-z)$ for $z =  \frac{\beta_t}{\omega_t}$, we obtain
\begin{eqnarray}
I_t(x^*) \le \frac{\tau(\frac{\beta_t}{\omega_t})}{ \tau(-\frac{\beta_t}{\omega_t})}EI_{t}(x^*)
\label{eq2.1}
\end{eqnarray}

This inequality (\ref{eq2.1}) holds with probability $ 1 -\delta$. Finally, we achieve
\begin{eqnarray*}
f(x^*) - \mu_t^+ & \le & I_t(x^*) \\
& \le & \frac{\tau(\frac{\beta_t}{\omega_t})}{ \tau(-\frac{\beta_t}{\omega_t})}EI_{t}(x^*) \\
& \le & \frac{\tau(\frac{\beta_t}{\omega_t})}{ \tau(-\frac{\beta_t}{\omega_t})}EI_{t}(x_{t}) \\
& \le & \frac{\tau(\frac{\beta_t}{\omega_t})}{ \tau(-\frac{\beta_t}{\omega_t})}( \text{max}\{0, f(x_{t}) - \mu_t^+\} + (\beta_t+ \omega_t)\sigma_{t-1}(x_{t})),
\end{eqnarray*}
where the first inequality holds by the definition of the function $I_t$. The second one comes from (\ref{eq2.1}). The third one holds by the property of the chosen point $x_{t} = \text{argmax}_{x} \alpha_t^{EI}(x)$. The final inequality hold due to Lemma \ref{lem2.1}.
\end{proof*}
\paragraph{Upper bounding Term 2.}
Bounding Term 2 is our important result. This is represented in the following lemma.
\begin{lemma}
Pick a $\delta \in (0, 1)$. Then with probability $1 - \delta$ we have $$\mu_{t}^+ - f(x^+_t) \le \frac{\beta_t}{\omega_t} (\sqrt{2\pi}(\beta_t + \omega_t)\sigma_{t-1}(x_{t}) + \sqrt{2\pi} \text{max}\{0,(f(x_{t}) - \mu_t^+) \}).$$
\label{lem2.3}
\end{lemma}
\begin{proof*}
By the definition of our GP-EI algorithm, $x^+_t =\text{argmax}_{x \in \mathcal D_t} \mu_{t-1}(x)$. It implies that $\mu_t(x^+_t) = \mu_{t}^+$. Hence,
\begin{eqnarray*}
\alpha_t^{EI}(x^+_t) & = & \omega_t \sigma_{t-1}(x^+_t) \tau(\frac{\mu_t(x^+_t) - \mu_{t}^+}{\omega_t\sigma_{t-1}(x^+_t)}) \\
& = & \omega_t \sigma_{t-1}(x^+_t) \tau(0) \\
& = & \frac{1}{\sqrt{2\pi}}\omega_t \sigma_{t-1}(x^+_t)
\end{eqnarray*}
Combining this with the fact that $\alpha_t^{EI}(x_{t}) = \text{max}_{x \in \mathcal X} \alpha_t^{EI}(x)$, we have
\begin{eqnarray*}
\frac{1}{\sqrt{2\pi}}\omega_t \sigma_t(x^+_t) & = & \alpha_t^{EI}(x^+_t) \\
& \le &  \alpha_t^{EI}(x_{t}) \\
& \le &   \text{max}\{0, f(x_{t}) - \mu_t^+\} + (\beta_t+ \omega_t)\sigma_{t-1}(x_{t})
\end{eqnarray*}
where the inequality holds due to Lemma \ref{lem2.1}.

Thus,
\begin{eqnarray}
\omega_t \sigma_{t-1}(x^+_t) \le (\sqrt{2\pi}(\beta_t + \omega_t)\sigma_{t-1}(x_{t+1}) + \sqrt{2\pi} \text{max}\{0,(f(x_{t}) - \mu_t^+) \})
\label{eq2.2}
\end{eqnarray}
Finally, we can upper bound $\mu_t^+ - f(x^+_t)$  with probability $1 - \delta$ based on (\ref{eq2.2}) as follows:
\begin{eqnarray*}
\mu_t^+ - f(x^+_t) & \le  & \beta_t \sigma_{t-1}(x^+_t) \\
& \le & \frac{\beta_t}{\omega_t} (\sqrt{2\pi}(\beta_t + \omega_t)\sigma_{t-1}(x_{t}) + \sqrt{2\pi} \text{max}\{0,(f(x_{t}) - \mu_t^+) \})
\end{eqnarray*}
where in the first inequality, $\mu_t^+ - f(x'_t) \le \beta_t \sigma_{t-1}(x'_t)$ with probability $1 - \delta$ due to Lemma \ref{le_beta_2}. $f(x'_t) - f(x^+_t) \le 0$ because by definition $x'_t \in \mathcal D_t$ and $x^+_t = \text{argmax}_{x_i \in \mathcal D_t} f(x_i)$. The second inequality holds with probability $1 -\delta$ due to Eq(\ref{eq2.2}).
Thus, the lemma holds with probability $1 -\delta$.
\end{proof*}
\paragraph{Upper bounding the instantaneous regret $r_t = f(x^*) - f(x_t^+)$.}
Combining Term 1 (in Lemma \ref{lem2.2}) and Term 2 (in Lemma \ref{lem2.3}), with high probability we have
$$r_t  \le (\frac{\tau(\frac{\beta_t}{\omega_t})}{ \tau(-\frac{\beta_t}{\omega_t})} + \sqrt{2\pi}\frac{\beta_t}{\omega_t}) [\text{max}\{0, f(x_{t}) - \mu_t^+\} + (\beta_t+ \omega_t)\sigma_{t-1}(x_{t})]$$.

Now we make this bound simple. By using the assumption of Theorem 2, $\omega_t = \sqrt{\gamma_{t-1} + 1 + ln(\frac{1}{\delta})}$, we can bound the ratio $\frac{\tau(\frac{\beta_t}{\omega_t})}{ \tau(-\frac{\beta_t}{\omega_t})}$ as follows:
\begin{lemma}
There exists $C > 0$ and $\frac{\tau(\frac{\beta_t}{\omega_t})}{ \tau(-\frac{\beta_t}{\omega_t})} \le C$  for every $ T \ge 1$ and $1 \le t \le T$.
\end{lemma}
\begin{proof*}
By definition, $\beta_t = B + \sqrt{2(\gamma_{t-1} + 1 + ln(1/\delta))}$. Hence, $\frac{\beta_t}{\omega_t} \le B + \sqrt{2}$. Since the function $\tau(z)$ is non-decreasing, we have that $\tau(\frac{\beta_t}{\omega_t}) \le \tau (B + \sqrt{2})$, and $\tau(-\frac{\beta_t}{\omega_t}) \ge \tau (-B - \sqrt{2})$. Thus,
$$\frac{\tau(\frac{\beta_t}{\omega_t})}{ \tau(-\frac{\beta_t}{\omega_t})} \le \frac{\tau(B + \sqrt{2})}{\tau(-
(B + \sqrt{2}))}.$$
Setting $C = \frac{\tau(B + \sqrt{2})}{\tau(-
(B + \sqrt{2}))}$ which is a constant, we have that $\frac{\tau(\frac{\beta_t}{\omega_t})}{ \tau(-\frac{\beta_t}{\omega_t})} \le C$  for every $ T \ge 1$ and $1 \le t \le T$. The lemma holds.
\end{proof*}

Thus, we obtain an upper bound on $r_t$ as follows:
\begin{lemma}
There exist constant $C > 0$ such that
$$r_t  \le  (C + \sqrt{2\pi}(B + \sqrt{2}))(I_t + (\beta_t+ \omega_t)\sigma_{t-1}(x_{t})),$$
where $I_t = \text{max}\{0, f(x_{t}) - \mu_t^+\}$.
\label{lem2.4}
\end{lemma}
\begin{proof}
\begin{eqnarray*}
r_t &\le &  (\frac{\tau(\frac{\beta_t}{\omega_t})}{ \tau(-\frac{\beta_t}{\omega_t})} + \sqrt{2\pi}\frac{\beta_t}{\omega_t}) [\text{max}\{0, f(x_{t}) - \mu_t^+\} + (\beta_t+ \omega_t)\sigma_{t-1}(x_{t})]\\
& \le & (C + \sqrt{2\pi}(B + \sqrt{2}))(I_t + (\beta_t+ \omega_t)\sigma_{t-1}(x_{t})),
\end{eqnarray*}
\end{proof}
\subsection{Upper bounding the sum $\sum_{t=1}^{T} r_t$}
Using Lemma \ref{lem2.4}, we obtain an upper bound for $\sum_{t=1}^{T} r_t$ as follows:
$$\sum_{t=1}^{T}r_t \le (C + \sqrt{2\pi}B + 2\sqrt{\pi}) \underbrace{\sum_{t=1}^{T-1}I_t}_{\text{Term 3}} + (C + \sqrt{2\pi}B + 2\sqrt{\pi}) \underbrace{\sum_{t=0}^{T-1}(\beta_t + \omega_t)\sigma_{t}(x_{t+1})}_{\text{Term 4}}.$$
To upper bound the sum, we will go to upper bound Term 3 and Term 4. While Term 4 can be bounded via the maximum information gain similar as in the existing works of GP-UCB and GP-TS \citep{Srinivas12,Chowdhury17}, bounding Term 3 is the key step in our proof for the acquisition function GP-EI.

Before providing upper bounds for Term 3 and Term 4, we restate some important results from existing works in the following section.
\subsubsection{Auxiliary Lemmas}
\sloppy
\begin{lemma}[Theorem 2 of \citep{Chowdhury17}]
Pick $\delta \in (0,1)$. Fix a horizon $T > 1$. We define $\beta_t = B + R\sqrt{2(\gamma_{t-1} + 1 + ln(1/\delta))}$ for every $1 \le t \le T$. Then
$$\mathbb{P}(\forall 1 \le t \le T, \forall x \in \mathcal X, |f(x) - \mu_{t-1}(x)| \le \beta_t \sigma_{t-1}(x) ) \ge 1 - \delta.$$
\label{le_beta_2}
\end{lemma}
\begin{lemma} [Lemma 3 of \citep{Chowdhury17}]
The information gain for the points selected can be expressed in terms of the predictive variances as follows:
$$I(y_T;f_T) = \frac{1}{2}\sum_{t=1}^{T} ln(1 + \lambda^{-1} \sigma_{t-1}(x_t)).$$
\label{lem_information_gain}
\end{lemma}
\begin{lemma}[Lemma 5 of \citep{Freitas12}]
When $f \in H_{k}(\mathcal X)$, then for every $x, y \in \mathcal X$, we have
$$|f(x) - f(y)| \le BL||x -y||_1,$$
where $L$ is a Lipschitz constant in $H_k(\mathcal X)$. \label{dimester}
\end{lemma}
\begin{proof}
Lemma 5 of \citep{Freitas12} holds for functions in an RKHS equipped by a kernel $k$. The Lipschitz constant is defined as $$\text{sup}{x \in \mathcal X} \partial_x \partial{x’} k(x-x’)|_{x’=x}$$, so the second derivative on $k$ is needed. For SE kernel, it alsways holds. For Mat\'ern kernel , to our best understanding, if we have $\nu \ge 1/2$, we can ensure the differentiability.
\end{proof}
We now are ready to bound Term 3 and Term 4.
\subsubsection{Upper Bounding Term 3}
\begin{lemma}
Pick $\delta \in (0,1)$. Then with probability at least $1 - \delta$ we have that
$$\sum_{t=1}^{T} I_t = \mathcal O(\beta_T  \sqrt{T\gamma_T}).$$ \label{lem_key}
\end{lemma}
\begin{proof*}
Set $S_T = \sum_{t=0}^{T-1} I_t$. There are three cases to be considered:
\paragraph{Case 1} $S_T = 0$. This happens when for every $t$: $f(x_{t+1}) - \mu_t^+ \le 0$.
\paragraph{Case 2} There exists an unique index $1 \le t'\le T$ such that $f(x_{t'+1}) - \mu_t'^+ > 0$. It follows that $S_T = f(x_{t'+1}) - \mu_t'^+$. In this case, we have that
\begin{eqnarray*}
S_T  & = & f(x_{t'+1}) - \mu_t'^+ \\
& \le & f(x_{t'+1}) - (f(x') - \beta_{t'+1}\sigma_{t'}(x')) \\
& \le & f(x_{t'+1}) - f(x') + \beta_{t'+1}\sigma_{t'}(x') \\
& \le &  BL ||x_{t' + 1} - x'||_1 + \beta_{t'+1}\\
& = & \mathcal O(\beta_T),
\end{eqnarray*}
where in the last inequality, we use Lemma \ref{le_b}, the inequality $\beta_{t'+1} \le \beta_T$, and the fact that $\sigma_t'(x) \le 1$. Finally, because the domain $\mathcal X$ is bounded, $||x_{t' + 1} - x'||_1$ is bounded.

\paragraph{Case 3} There are $0 \le t_1 < t_2, ..., < t_l \le T-1$ where $l \ge 2$ such that $f(x_{t_i +1}) \ge \mu_{t_i}^+$. Thus, we have
\begin{eqnarray*}
\sum_{t=0}^{T-1} I_t & = & \sum_{t=1}^{T} \text{max}\{0, f(x_{t+1}) - \mu_t^+\} \\
& = & \sum_{i = 1}^{l} (f(x_{t_i +1}) - \mu_{t_i}^+) \\
& \le & \sum_{i = 1}^{l} (\beta_{t_i+1} \sigma_{t_i}(x_{t_i +1}) + \mu_{t_i}(x_{t_i +1}) - \mu_{t_i}^+) \\
& \le & \underbrace{\sum_{i = 1}^{l} \beta_{t_i+1} \sigma_{t_i}(x_{t_i +1})}_{\text{Term 5}} +  \underbrace{\sum_{i = 1}^{l}(\mu_{t_i}(x_{t_i +1}) - \mu_{t_i}^+)}_{\text{Term 6}}
\end{eqnarray*}
\paragraph{Bound Term 5}
\begin{eqnarray*}
\sum_{i = 1}^{l} \beta_{t_i+1} \sigma_{t_i}(x_{t_i +1}) & \le & \sum_{t=0}^{T-1} \beta_{t+1} \sigma_{t}(x_{t+1}) \\
& \le &  \beta_T \sum_{t=0}^{T-1} \sigma_{t}(x_{t+1})
\end{eqnarray*}
\paragraph{Bound Term 6} Set $M_1 = \sum_{i = 1}^{l}(\mu_{t_i}(x_{t_i +1}) - \mu_{t_i}^+)$.
\begin{eqnarray*}
& M_1 & =\mu_{t_l}(x_{t_l +1}) - \mu_{t_1}^+  + \sum_{i=1}^{l-1} (\mu_{t_{i-1}}(x_{t_{i-1} +1}) - \mu_{t_i}^+))\\
&\le & \underbrace{\mu_{t_l}(x_{t_l}) - \mu_{t_1}^+}_{\text{Term 7}}  + \underbrace{\sum_{i=1}^{l-1} (\mu_{t_{i-1}}(x_{t_{i-1} +1}) - \mu_{t_i}(x_{t_{i-1} +1})))}_{\text{Term 8}}\\
\end{eqnarray*}
\paragraph{Bound Term 7} Set $M_2 = \mu_{t_l}(x_{t_l+1}) - \mu_{t_1}^+ $. We have
\begin{eqnarray*}
M_2 & \le & f(x_{t_l + 1}) + \beta_{t_l+1}\sigma_{t_l}(x_{t_l+1}) - (f(x_{t_1} - \beta_{t_1}\sigma_{t_1}(x_{t_1}))) \\
& \le &  f(x_{t_l + 1}) - f(x_{t_1}) + \beta_{t_l+1}\sigma_{t_l}(x_{t_l+1})+ \beta_{t_1 +1}\sigma_{t_1}(x_{t_1})) \\
& \le & f(x_{t_l + 1}) - f(x_{t_1}) + \beta_{t_l +1}  + \beta_{t_1}  \\
& \le & BL||x_{t_l + 1} - x_{t_1}||_1 + 2\beta_T  \\
& \le & \mathcal O(\beta_T)
\end{eqnarray*}
The argument to achieve the bound for Term 7 is similar to Case 2.
\paragraph{Bound Term 8}
Set $M_3 = \sum_{i=1}^{l-1} (\mu_{t_{i-1}}(x_{t_{i-1} +1}) - \mu_{t_i}(x_{t_{i-1} +1})))$ for simplicity. We go to bound $M$.
\begin{eqnarray*}
& M_3 &  \le   \sum_{i=1}^{l-1} (f(x_{t_{i-1} +1}) + \beta_{t_{i-1}+1}\sigma_{t_{i-1}}(x_{t_{i-1} +1}))  \\
&  & - (f(x_{t_{i-1} +1})- \beta_{t_{i}+1}\sigma_{t_{i}}(x_{t_{i-1} +1})) \\
& = & \sum_{i=1}^{l-1} \beta_{t_{i-1}+1}\sigma_{t_{i-1}}(x_{t_{i-1} +1}) + \beta_{t_{i}+1}\sigma_{t_{i}}(x_{t_{i-1} +1}) \\
& \le & \sum_{i=1}^{l-1}(\beta_{t_{i-1}+1} + \beta_{t_{i}+1})\sigma_{t_{i-1}}(x_{t_{i-1} +1})\\
& \le & 2\beta_T \sum_{i=1}^{l-1}\sigma_{t_{i-1}}(x_{t_{i-1} +1})\\
& \le & 2\beta_T \sum_{i=0}^{T-1} \sigma_{i}(x_{i+1}),
\end{eqnarray*}
where in the first inequality, we use Lemma \ref{le_a}: $\mu_{t_{i-1}}(x_{t_{i-1} +1}) \le f(x_{t_{i-1} +1}) + \beta_{t_{i-1}}\sigma_{t_{i-1}}(x_{t_{i-1} +1})$; $\mu_{t_i}(x_{t_{i-1} +1}) \ge f(x_{t_{i} +1}) - \beta_{t_{i}}\sigma_{t_{i}}(x_{t_{i-1} +1})$. In the second inequality, we use the fact that $f(x_{t_{i-1} +1}) \le f(x_{t_{i} +1})$. In the third inequality, we use the \emph{decreasing monotonicity of variance functions} (\citep{vivarelli} and \citep{Chowdhury17}, see Section F). Here, we use $$\sigma_{t_{i}}(x_{t_{i-1} +1}) \le \sigma_{t_{i-1}}(x_{t_{i-1} +1}),$$ because $t_i > t_{i-1}$.
This step is crucial to bound $M_3$. Without this step, $M_3$ may be bounded by two sums:  $\sum^{l-1}_{i=1}  \sigma_{t_{i-1}}(x_{t_{i-1} +1}))$ and $ \sum^{l-1}_{i=1} \sigma_{t_{i}}(x_{t_{i-1} +1}))$. While the first term can be bounded in terms of the information gain, bounding the second  term is very challenging.

For every $x_i$, where $ 1 \le i \le T-1$, Lemma \ref{le_a} holds with probability $1 -\delta$. Therefore, Lemma \ref{le_a} holds with probability at least $1 -\delta$ for all $x_i$, where $ 1 \le i \le T-1$.
Combining Term 5, Term 7, Term 8, with probability $1 -T\delta$ we have
$$\sum_{t=0}^{T-1} \text{max}\{0, f(x_{t+1}) - \mu_t^+\} \le  \mathcal O(\beta_T \sum_{i=1}^{T-1} \sigma_{i}(x_{i+1})).$$

On the other hand, following Lemma 4 of \cite{Chowdhury17}, we have $\sum_{i=1}^{T-1} \sigma_{i}(x_{i+1}) \le  \sqrt{4(T+2)\gamma_T}$.  Thus,
$$\sum_{t=0}^{T-1} \text{max}\{0, f(x_{t+1}) - \mu_t^+\} = \mathcal O(\beta_T \sqrt{T\gamma_T}).$$
\end{proof*}
\subsubsection{Upper Bounding Term 4}
\begin{lemma}
Let $x_1, ..., x_t$ be the points selected by Algorithm 2. The sum of predictive standard deviation at those points can be expressed in terms of the maximum information gain. More precisely,
$$\sum_{t=0}^{T-1} \sigma_{t}(x_{t+1}) \le  \sqrt{4(T+1)\gamma_T}.$$
\label{lem_gain}
\end{lemma}
\begin{proof*}
By Cauchy-Schwartz inequality, $\sum_{t=0}^{T-1} \sigma_{t}(x_{t+1}) \le \sqrt{T\sum_{t=1}^{T-1} \sigma^2_{t}(x_{t+1})}$. By assumption, $0\le k(x,x) \le 1$. It implies that $0 \le \sigma_{t-1}^2(x) \le 1$ for all $x \in \mathcal X$. Following Lemma 4 of \citep{Chowdhury17}, we get
$\sum_{t=1}^{T} \sigma_{t-1}(x_t) \le \sqrt{4T(1 + 2/T)\gamma_T} \le \sqrt{4(T+2)\gamma_T}$. Note that we here use $\lambda = 1 + 2/T$ as the setting of \citep{Chowdhury17}.

Thus, we have that $\sum_{t=0}^{T-1} \sigma_{t}(x_{t+1})  \le \sqrt{4(T + 2)\gamma_T}$
\end{proof*}
Combining Lemma \ref{lem_key} and Lemma \ref{lem_gain}, we obtain an upper bound for the cumularive regret $\sum_{t=1}^{T}r_t$ as follows:
\begin{theorem}
Pick $\delta \in (0,1)$. Then with probability at least $1 - \delta$, the cumulative regret of Algorithm 1 is bounded as:
$$R_T = \mathcal O(\gamma_T\sqrt{T}).$$
\end{theorem}
\section{Proof of Theorem 2}
In this section, we provide a complete proof for Theorem 2 in the main paper. Here we use a non-trivial combination of the proof techniques as above and the technical results for $\pi$-GP-UCB \citep{janz20a}. A key difference from \citep{janz20a} lies in steps from equation 16 to 21 to obtain an upper bound for $R_T^+$ which is easy in \citep{janz20a}.

To obtain an upper bound the sum of $R_T = \sum_{t=1}^{T} r_t$, where $r_t = f(x^*) - f(x^+_t)$,  we denote $\tilde{\mathcal A}_T = \cup_{t \le T} \mathcal A_t$, the set of all cover elements created until time $T$, and define the initial time for an element $A \in \tilde{\mathcal A}_T $, as $\phi(A) = \text{min}\{t: A \in \mathcal A_t\}$, and the terminal time as $\phi'(A) = \text{max}\{t: A \in \mathcal A_t\}$.

Unlike the analysis of the sum  $R_T = \sum_{t=1}^{T} r_t$ in the proof of Theorem 1, we here analyze $R_T = \sum_{t=1}^{T} r_t$ into groups based on the partitioning of the searching space.
\begin{eqnarray}
R^+_T  & = & \sum_{t=1}^{T} r_t \\
& = & \sum_{t=1}^{T} \sum_{A \in \mathcal A_t} \textbf{1}\{x_t \in A\}r_t \\
& = & \sum_{A \in \tilde{\mathcal A}_T} \sum_{t =\phi(A)}^{\phi'(A)} \textbf{1}\{x_t \in A\}r_t \label{eq3.00}
\end{eqnarray}

The next step is to estimate each component $\sum_{t =\phi(A)}^{\phi'(A)} \textbf{1}\{x_t \in A\}r_t$. Given a set $A \in \mathcal A_T$, we consider the set $\mathcal D_T \cap A$. Without loss of generality, assume that $\mathcal D_T \cap A = \{x^A_1,..., x^A_{n(A)}\}$ with $n(A) = |\mathcal D_T \cap A|$. We assume that this order is the real order of time that $x^A_i$ is selected by the algorithm. It means that $x^A_1$ corresponds to time $t_1$, ..., and $x^A_{n(A)}$ corresponds to time $t_{n(A)}$
where $1 \le t_1 < t_2, <...< t_{n(A)} \le  T$. In addition, for every $1 \le i \le n(A)$, $i \le t_i$. Let $x^{A+}_i = \text{argmax}_{x^A_j \in \mathcal D_T \cap A, 1\le j \le i} f(x^A_j)$ and let $r^A_i = f(x^*) - f(x^{A+}_i)$ as the instantaneous regret at iteration $i$ for the sampled points in $A$.  We have that
\begin{eqnarray}
\sum_{t =\phi(A)}^{\phi'(A)} \textbf{1}\{x_t \in A\}r_t & = & \sum_{i=1}^{n(A)} r_{t_i} \\
& = & \sum_{i=1}^{n(A)} (f(x^*)   - f(x^+_{t_i})) \\
& \le & \sum_{i=1}^{n(A)} (f(x^*) - f(x^{A+}_i)) \\
& = & \sum_{i=1}^{n(A)}r^A_i,
\end{eqnarray}
where in Eq (12), we use the fact the  $f(x^+_{t_i}) \ge f(x^{A+}_i)$. Indeed, we have that $x^{A+}_i) = \text{argmax}_{x^A_j \in \mathcal D_T \cap A, 1\le j \le i} f(x^A_j)$ which is defined as above. $x_{t_i}$ corresponds to $x^A_i$ in set $A$. Therefore, $\{x^A_1, ..., x^A_i \}$ is the subset of $\{x_1, x_2, ..., x_{t_i}\}$. It implies that $f(x^+_{t_i}) \ge f(x^{A+}_i)$.

By Lemma \ref{lem3.5} which we will provide in section C.1, we have that
\begin{eqnarray}
\sum_{i=1}^{n(A)}r^A_i \le \tilde{\mathcal O}( \sqrt{n(A)\gamma^A_{\phi'(A)}}),
\label{eq3.0}
\end{eqnarray}
where $\gamma^A_{\phi'(A)}$ is the information gain of $A$ at iteration $\phi'(A) = \text{max}\{t: A\in \mathcal A_t\}$ (See Lemma \ref{lem_gain_new} in  Section C.1).

Thus, $\sum_{t =\phi(A)}^{\phi'(A)} \textbf{1}\{x_t \in A\}r_t \le \tilde{\mathcal O}(\sqrt{n(A)\gamma^A_{\phi'(A)}})$. It is equivalent that
\begin{eqnarray}
\frac{(\sum_{t =\phi(A)}^{\phi'(A)} \textbf{1}\{x_t \in A\}r_t)^2}{n(A)} \le \tilde{\mathcal O}(\gamma^A_{\phi'(A)})
\label{eq3.1}
\end{eqnarray}

Since the inequality at Eq(\ref{eq3.1}) holds for every $A \in \tilde{\mathcal A}_T$, we get
\begin{eqnarray}
\sum_{A \in \tilde{\mathcal A}_T} \frac{(\sum_{t =\phi(A)}^{\phi'(A)} \textbf{1}\{x_t \in A\}r_t)^2}{n(A)} \le \sum_{A \in \tilde{\mathcal A}_T} \tilde{\mathcal O}(\gamma^A_{\phi'(A)})
\label{eq3.2}
\end{eqnarray}

On the other hand, applying the Cauchy–Schwarz inequality, we have that
\begin{eqnarray}
(\sum_{A \in \tilde{\mathcal A}_T} n(A)(\sum_{A \in \tilde{\mathcal A}_T} \frac{(\sum_{t =\phi(A)}^{\phi'(A)} \textbf{1}\{x_t \in A\}r_t)^2}{n(A)}) \ge (\sum_{A \in \tilde{\mathcal A}_T} \sum_{t =\phi(A)}^{\phi'(A)} \textbf{1}\{x_t \in A\}r_t)^2
\label{eq3.3}
\end{eqnarray}
As the sets $A$ in $\tilde{\mathcal A}_T$ are not overlapping, $\sum_{A \in \tilde{\mathcal A}_T} n(A) = T$. Replacing this to Eq(\ref{eq3.3}), we get
\begin{eqnarray}
\sqrt{T(\frac{(\sum_{t =\phi(A)}^{\phi'(A)} \textbf{1}\{x_t \in A\}r_t)^2}{n(A)})} \ge \sum_{A \in \tilde{\mathcal A}_T} \sum_{t =\phi(A)}^{\phi'(A)} \textbf{1}\{x_t \in A\}r_t
\label{eq3.4}
\end{eqnarray}

Combining Eq(\ref{eq3.00}), Eq(\ref{eq3.2}) and Eq(\ref{eq3.4}), we get that
\begin{eqnarray}
R^+_T  &=& \sum_{A \in \tilde{\mathcal A}_T} \sum_{t =\phi(A)}^{\phi'(A)} \textbf{1}\{x_t \in A\}r_t \\
& \le & \sqrt{T (\sum_{A \in \tilde{\mathcal A}_T}\tilde{\mathcal O}(\gamma^A_{\phi'(A)}))}
\label{eq3.5}
\end{eqnarray}
By Lemma \ref{lem3.6} which is proven in Section C.1, the number of the sets $A \in \tilde{\mathcal A}_T$ is $\mathcal O(T^q)$, where $q = \frac{d(d+1)}{d(d+2) + 2\nu}$ (See Section 5.1). By Lemma \ref{lem_gain_new} in Section C.1, $\gamma^A_{\phi'(A)}$ is bounded by a logarithmic function in $T$ for all $A \in \tilde{\mathcal A}_T$. Thus, we have $R_T \le \mathcal O(T^{\frac{q+1}{2}})$. Using the definition of $q = \frac{d(d+1)}{d(d+2) + 2\nu}$, we get that $R_T = \tilde{\mathcal O}(T^{\frac{-2\nu -d}{d(2d+4) + 4\nu}})$.

\subsection{Auxiliary Lemmas}
\begin{lemma}
For every $1 \le i  \le T$, we have that $\gamma_i \le \gamma_T$. \label{lem_gamma}
\end{lemma}
\begin{proof*}
By Lemma \ref{lem_information_gain}, for a set of point $A_i = {x_1 ,x_2, ..., x_i}$, we have that $I(y_i;f_i) = \frac{1}{2}\sum_{t=1}^{i} ln(1 + \lambda^{-1} \sigma_{t-1}(x_t))$. Consider a set $A_T = A_i \cup \{x_{i+1, ..., x_T}\}$ containing $T$ elements. We also have that $I(y_{A_T};f_{A_T}) = \frac{1}{2}\sum_{t=1}^{T} ln(1 + \lambda^{-1} \sigma_{t-1}(x_t))$. Thus, for each $A_i$ which contains $i$ elements, we always can construct a set $A_T$ containing $T$ elements such that $I(y_i;f_i) \le I(y_T;f_T)$.
Thus, $$\gamma_i = \text{max}_{A \in \mathcal X: |A| = i} I(y_A: f_A) \le \text{max}_{A \in \mathcal X: |A| = T} I(y_A: f_A) = \gamma_T.$$
\end{proof*}
\begin{lemma}[Lemma 1 of \citep{janz20a}]
Let $A$ be a subset of $\mathcal X$ and assume that there exists a $1 \le \phi \le T$ such that $A \in \mathcal A_{\tau}$. Let $\phi'(A) = \text{max}\{1 \le t \le T: A \in \mathcal A_t\}$. Then for some $C > 0$, $\gamma_{\phi'(A)}^A \le Cln(T)lnln(T)$.
\label{lem_gain_new}
\end{lemma}
\begin{lemma}[Lemma 2 of \citep{janz20a}]
Let A be the covering set at time $t$. Assume that $\mathcal A_2 = \mathcal O(T^q)$. Then for $T$ sufficiently large, $|\cup_{t \le T} \mathcal A_t| \le C_{d,v} T^q$, where $C_{d,v} > 0$ depends on $d$ and $\nu$ only.
\label{lem3.6}
\end{lemma}

\begin{lemma}[Lemma 5 of \citep{janz20a}]
Given $\delta \in (0,1)$. for all $t \le T$, for all $A \in \cup_{t \le T} \mathcal A_t$, we have
$$\mathbb{P}(\forall t, \forall x \in A, |f(x) - \mu^A_{t-1}(x)| \le \hat{\beta}^A_t \sigma_{t-1}(x) ) \ge 1 - \delta,$$
where $\hat{\beta}^A_t = B + R \sqrt{2(\gamma^A_{t-1} + 1 + ln(1/\delta))}$, and $\gamma^A_{t-1} = \frac{1}{2} ln|I + \lambda^{-1}K^A_{t-1}|$.
\label{le_beta_new}
\end{lemma}

Before proving the important Lemma \label{lem3.5},  we need the following lemma to upper bound $\frac{\tau(\frac{\beta^{A}_{t_i}}{\omega_T})}{ \tau(-\frac{\beta^{A}_{t_i}}{\omega_T})}$ for $1 \le i \le n(A)$.
\begin{lemma}
Given $\omega_T = \sqrt{\text{ln}(T)\text{ln} \text{ln}(T)}$. Let $A$ be a subset of $\mathcal X$ and assume that there exists a $1 \le \phi \le T$ such that $A \in \mathcal A_{\tau}$. Let $\phi'(A) = \text{max}\{1 \le t \le T: A \in \mathcal A_t\}$. There exists a constant $C > 0$ such that for every $T > 0$ and for every $1 \le i \le n(A)$ we have
$$\frac{\tau(\frac{\beta^A_{t_i}}{\omega_T})}{ \tau(-\frac{\beta^A_{t_i}}{\omega_T})} \le C.$$
\end{lemma}
\begin{proof}
By definition, $\beta^A_{t_i} = B + \sqrt{\gamma^A_{t_i -1} + 1 + ln(1/\delta)}$. Using the increasing monotonicity of function $\gamma^A$, we have that $\gamma^A_{t_i -1} \le \gamma_{\phi'(A)}^A$ (because $t_i-1 \le \phi'(A)$). By Lemma \ref{lem_gain_new}, $\gamma_{\phi'(A)}^A \le Cln(T)lnln(T)$. Therefore, $\beta^A_{t_i} = B + \sqrt{\gamma^A_{t_i -1} + 1 + ln(1/\delta)} \le B + \sqrt{\gamma_{\phi'(A)}^A + 1 + ln(1/\delta)} \le B + \sqrt{Cln(T)lnln(T) + 1 + ln(1/\delta)}$.

On the other hand, there exists a $C_1$ large enough such that $B + \sqrt{Cln(T)lnln(T) + 1 + ln(1/\delta)} \le C_1\sqrt{ln(T)lnln(T)}$. Combining this with the above result, we imply that $\beta^A_{t_i} \le C_1\sqrt{ln(T)lnln(T)} = C_1\omega_T$. It is equivalent that $\frac{\beta^A_{t_i}}{\omega_T} \le C_1$.

Since the function $\tau(z)$ is non-decreasing, we have that $\tau(\frac{\beta^A_{t_i}}{\omega_T}) \le \tau ( C_1)$, and $\tau(-\frac{\beta^A_{t_i}}{\omega_T}) \ge \tau (- C_1)$. Thus,
$$\frac{\tau(\frac{\beta^A_{t_i}}{\omega_T})}{ \tau(-\frac{\beta^A_{t_i}}{\omega_T})} \le \frac{\tau( C_1)}{\tau(- C_1)}.$$
Set $C = \frac{\tau( C_1)}{\tau(- C_1)}$. We now can bound the ratio $\frac{\tau(\frac{\beta^A_{t_i}}{\omega_T})}{ \tau(-\frac{\beta^A_{t_i}}{\omega_T})}$ by a constant.
$$\frac{\tau(\frac{\beta^A_{t_i}}{\omega_T})}{ \tau(-\frac{\beta^A_{t_i}}{\omega_T})} \le C.$$
\end{proof}

Finally, we provide an upper bound for the sum $\sum_{i=1}^{n(A)}r^A_i$.
\begin{lemma}
Let $\omega_T = \sqrt{\text{ln}(T)\text{ln} \text{ln}(T)}$. Given a set $A \in \mathcal A_t$ and assume that there exists a $1 \le \phi \le T$ such that $A \in \mathcal A_{\tau}$. Let $\phi'(A) = \text{max}\{1 \le t \le T: A \in \mathcal A_t\}$. Then with probability at least $1 - \delta$, we have
$$\sum_{i=1}^{n(A)}r^A_i \le \tilde{\mathcal O}(\sqrt{n(A)\gamma^A_{\phi'(A)}}),$$
where $n(A) = |\mathcal D_T \cap A|$ which is defined as above. The notation $\tilde{\mathcal O}$ is a
variant of $\mathcal O$, where log factors are suppressed.
\label{lem3.5}
\end{lemma}
\begin{proof}
Following the Improved-GP-EI algorithm, each $A \in \mathcal A_t$ is fitted by an independent Gaussian process with $N(A)$ observations and the sampled points in $\mathcal D_t^A$: $(x_1^A,y_1^A), ..., (x_{N(A)}^A,y_{N(A)}^A)$. Therefore, we can apply the results in Section B to the set $A$.

Similar to the proofs of Lemma 6, Lemma 7, and Lemma 8, we obtain an upper bound on $r^A_i = f(x^*) - f(x^{A+}_i)$ as follows.
$$r^A_i \le (\frac{\tau(\frac{\beta^{A}_{t_i}}{\omega_T})}{ \tau(-\frac{\beta^{A}_{t_i}}{\omega_T})} + \sqrt{2\pi}) \text{max}\{0, f(x^A_{i+1}) - \mu^{A+}_{t_i}\} + (\frac{\tau(\frac{\beta^{A}_{t_i}}{\omega_T})}{ \tau(-\frac{\beta^{A}_{t_i}}{\omega_T})}(\beta^{A}_{t_i} + \omega_T) + \sqrt{2\pi}(3\beta^{A}_{t_i} + \omega_T))\sigma^A_{t_i}(x^A_{i+1})),$$
where $\mu^{A+}_i = \text{max}_{x^A_j \in \mathcal D_T \cap A, 1\le j \le i} \mu_{t_i}(x^A_j)$. Recall that $x^{A+}_i = \text{argmax}_{x^A_j \in \mathcal D_T \cap A, 1\le j \le i} f(x^A_j)$.

Using Lemma 21, we can bound the ratio $\frac{\tau(\frac{\beta^A_{t_i}}{\omega_T})}{ \tau(-\frac{\beta^A_{t_i}}{\omega_T})}$ by a constant which is independent of $1 \le t\le T$ and $T$. Thus, by the proofs similar to those of Lemma 14 and Lemma 15, we obtain an upper bound on the sum $\sum_{i=1}^{n(A)}r^A_i$ as
$$\sum_{i=1}^{n(A)}r^A_i \le \mathcal O(\beta^A_{t_{N(A)}}\sqrt{n(A)\gamma^A_{\phi'(A)}}).$$
Using the proof similar to Lemma 21 as above, there exists constant $C_1$ such that $\beta^A_{t_{N(A)}} \le C_1\omega_T = \Theta(\sqrt{\text{ln}(T)\text{ln} \text{ln}(T)})$. We can remove this log component from the upper bound. Finally, we have that
$$\sum_{i=1}^{n(A)}r^A_i \le \tilde{\mathcal O}(\sqrt{n(A)\gamma^A_{\phi'(A)}}).$$
\end{proof}

\section{Additional Experiments for Real-World Benchmarks}
We now take a 2-layer perceptron network (MLP) with 512 neurons/layer and optimize three hypeparameters: the learning rate $l$ and the norm regularization hyperparameters $l_{r1}$ and $l_{r2}$ of the two layers. We train the algorithms using the MNIST train dataset (55000 patterns) and then test the model on the MNIST test dataset (10000 patterns). We plot the optimization results using  prediction accuracy in Figure \ref{fig3}.

\begin{figure}
\begin{center}
\centerline{\includegraphics[scale=1.0,width=.5\textwidth]{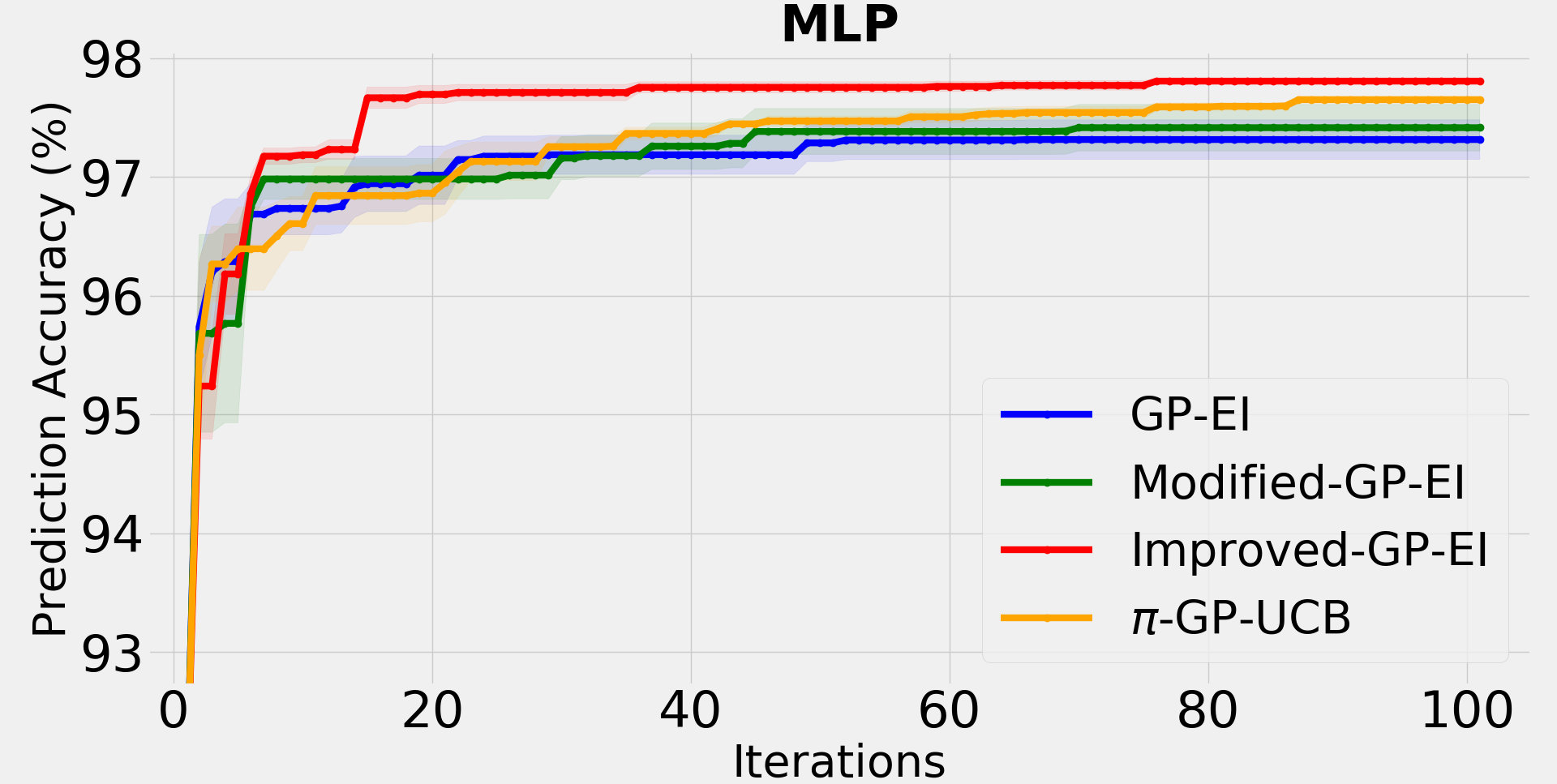}}
\end{center}
\caption{Comparison of methods on MLP.}
\label{fig3}
\end{figure}

\end{document}